\PassOptionsToPackage{x11names,table}{xcolor}
\documentclass[]{fairmeta}
\usepackage{xcolor}
\usepackage{colortbl}
\usepackage{caption}
\usepackage[linesnumbered, ruled]{algorithm2e}
\usepackage[x11names,table]{xcolor}
\usepackage[normalem]{ulem}
\usepackage{mathtools}
\usepackage{mathrsfs}
\newtheorem{theorem}{Theorem}[section]

\newtheorem{lemma}[theorem]{Lemma}

\theoremstyle{definition}
\newtheorem{definition}[theorem]{Definition}
\theoremstyle{remark}
\newtheorem{remark}{Remark}[section]
\theoremstyle{remark}

\theoremstyle{remark}

\theoremstyle{remark}

\theoremstyle{remark}

\theoremstyle{remark}

\newtheorem{assumption}{Assumption}[section]
\theoremstyle{remark}

\newcommand{\smallcitep}[1]{{\scriptsize \citeauthor{#1}}}

\title{\textsc{LlamaRL}: A Distributed Asynchronous Reinforcement Learning Framework for Efficient Large-scale LLM Training}

\author[1]{Meta GenAI}
\affiliation[1]{A detailed contributor list can be found in the appendix of this paper.}

\abstract{
Reinforcement Learning (RL) has become the most effective post-training approach for improving the capabilities of Large Language Models (LLMs). In practice, due to the high demands on latency and memory, it is particularly challenging to develop an efficient RL framework that reliably manages policy models with hundreds to thousands of billions of parameters. 

In this paper, we present \textsc{LlamaRL}, a fully-distributed, asynchronous RL framework optimized for efficient training of large-scale LLMs with various model sizes (8B, 70B and 405B parameters) on GPU clusters with a handful to thousands devices. \textsc{LlamaRL} introduces a streamlined, single-controller architecture built entirely on native PyTorch, enabling modularity, ease of use, and seamless scalability to thousands of GPUs. We also provide a theoretical analysis of \textsc{LlamaRL}’s efficiency, including a formal proof that its asynchronous design leads to strict RL speed-up. Empirically during the Llama 3 post training, by leveraging best practices such as co-located model offloading, asynchronous off-policy training, and distributed direct memory access for weight synchronization, \textsc{LlamaRL} achieves significant efficiency gains—up to 10.7x speedup compared to the DeepSpeed-Chat like systems on a 405B-parameter policy model. Furthermore, the efficiency advantage continues to grow with increasing model scale, demonstrating the framework's suitability for future large-scale RL training.}

\date{\today}

\begin{document}

\maketitle

\section{Introduction}
\label{section:intro}

Reinforcement Learning (RL) has become the most effective post-training approach for improving the capabilities of Large Language Models (LLMs) \citep{Ouyang2022RLHF,jaech2024openai,deepseek2025r1}. In practice, due to the high demands on latency and memory, it is particularly challenging to develop an efficient RL framework that reliably manages policy models with hundreds to thousands of billions of parameters. 

In this paper, we present \textsc{LlamaRL}, a fully-distributed, asynchronous reinforcement learning (RL) framework optimized for efficient training of large-scale language models (LLMs) with different sizes (8B, 70B and 405B parameters,  \citep{grattafiori2024llama}) on GPU clusters with a handful to thousands H100 GPUs. \textsc{LlamaRL} introduces a streamlined, single-controller architecture built entirely on native PyTorch, enabling modularity, ease of use, and seamless scalability to thousands of GPUs. We also provide a theoretical analysis of \textsc{LlamaRL}’s efficiency, including a formal proof that its asynchronous design leads to strict RL speed-up. Empirically, by leveraging best practices such as co-located model offloading, asynchronous off-policy training, and GPU-native distributed weight synchronization via NVLink, \textsc{LlamaRL} achieves significant efficiency gains—up to 10.7x speedup compared to state-of-the-art systems (e.g., DeepSpeed-Chat \citep{yao2023deepspeedchateasyfastaffordable}) on a 405B-parameter policy model. Furthermore, the efficiency advantage continues to grow with increasing model scale, demonstrating the framework's suitability for future large-scale RL training.

\subsection{Challenges of Scaling RL for LLM}
Started as a relatively lightweight calibration against human preferences \citep{Christiano2017,Ouyang2022RLHF}, RL has recently proved effective at enabling powerful capabilities such as reasoning \citep{jaech2024openai,deepseek2025r1}, code generation \citep{gehring2024rlef,wei2025swe}, among others. Therefore, an efficient, succinct and highly-scalable RL training framework would accelerate research progress and become beneficial to a bigger community. However, implementing such a training framework poses significant engineering challenges. These challenges can be summarized into three main categories:

\paragraph{\textbf{Flexibility to Support Various RL Algorithms.}} RL training process is inherently complex, not only due to the allocation of models and associated computations within the algorithm but also because of the diverse data communication among them. To date, a number of RL algorithms have been developed to finetune LLMs. PPO requires running a pipeline involving four models, i.e., actor, critic, reference policy and reward models, interacting in a sophisticated manner during training. Representative algorithms such as RAFT~\citep{dong2023raft}, ReMax \citep{li2024remaxsimpleeffectiveefficient}, RLOO \citep{ahmadian2024back} and GRPO \citep{shao2024deepseekmath} have been developed without using the critic model. Multi-objective optimization frameworks such as CGPO \citep{xu2024perfect} has been developed to leverage multiple reward models and judges at the same time. As the complexity of these algorithms grows, it is important to build a flexible training framework which can easily extend to different number of models and schedule the corresponding data and parameter flows.

\paragraph{\textbf{Scaling to Larger Models and More GPUs.}} Training larger models consistently yields enhanced performance across a wide range of tasks  ~\citep{brown2020language,Kaplan2020Scaling,hoffmann2022training}. Accommodating these ever-larger models typically requires intricate parallelism to handle both memory and computational constraints, it becomes more challenging in RL as various actors (e.g. policy model, reward model etc.) with varying sizes run simultaneously on large GPU clusters. Taking the Llama 3.1 models \citep{grattafiori2024llama} as an example, training the 405B model via PPO requires: (1) A minimum of 512  NVIDIA's H100 GPUs in total; (2) A model tensor parallelism (TP) size of 32 is used to reduce the memory footprint of model weights on each GPU; (3) A Fully Sharded Data Parallel (FSDP) optimization with a parallelism shard of 16 is used to further reduce the memory footprint \citep{rajbhandari2019zero,zhao2023pytorchfsdpexperiencesscaling}; (4) A smaller LLM model (e.g. 8B parameters) acts as a reward model. Therefore, a highly scalable RL framework is required. It must efficiently support diverse and large-scale model sizes, as well as flexible parallelism strategies, across different actors and GPU configurations to ensure efficient large-scale training.

\paragraph{\textbf{Insufficient GPU Usage due to Idle Bubbles.}} 
For LLMs, RL algorithms need to run response generation followed by loss computation and weight updates. Because different responses on different workers can have different lengths, taking different amount of time to process, such sequential execution often result in bubbles where some GPU workers are idle. The situation of insufficient GPU usage becomes even worse when the larger models introduce larger generation time differences causing larger bubbles. In some cases, slow data flow communication is another issue to waste the GPU resource. Implementing a RL training framework with high GPU utilization can not only shorten the fine-tuning wall-clock time, but also pave a way for the the bigger community to train a model with less GPU resources. Additionally, maintaining high GPU utilization while scaling both model size and GPU resources remains a significant challenge.

\subsection{Contributions}
To address these challenges, we introduce \textsc{LlamaRL}, a fully-distributed, asynchronous RL training framework optimized for efficient, large-scale LLM training. Our key contributions are summarized below, following a brief background and related work discussion:
\begin{itemize}
    \item \textbf{Simple and Modular Architecture:} \textsc{LlamaRL} employs a streamlined single-controller architecture built entirely on native PyTorch. This design can seamlessly scale across thousands of GPUs, facilitating efficient training of large-scale LLMs (e.g., a 405B-parameter policy model). Its modular structure and intuitive control logic enable users to easily adapt and extend the framework, supporting diverse RL algorithms. We detail such designs in Section~\ref{sec:design}.
    \item \textbf{High Efficiency with Scalable Best Practices:} \textsc{LlamaRL} introduces a set of best practices specifically tailored for RL training of large-scale and continuously scaling LLMs. Our experiments demonstrate a 10.7x efficiency improvement compared to state-of-the-art systems such as DeepSpeed-Chat on a 405B-parameter model, with this efficiency advantage growing further as model scale increases. The core best practices include the following, as we will detail in Section~\ref{sec:strategies}:
    \begin{itemize}
        \item \textbf{Co-located Model Offloading with Fine-grained Parallelism and Quantization:} We completely offload the generation process from the training cluster, recognizing that it is memory-bound and significantly contributes to execution time. Complete offloading enhances scalability and flexibility, enabling fine-grained parallelism and quantization optimizations, thus significantly reducing compute and memory requirements.
        \item \textbf{Asynchronous Off-policy RL:} To minimize computational bubbles, our training and generation processes run asynchronously in parallel. This approach significantly improves throughput and resource utilization, hence speeding up the overall training. Inevitably, this introduces off-policyness with 1 to \textit{n} steps of delay, as illustrated in Figure~\ref{fig:example_rlhf_workflow}. We introduce \texttt{AIPO} - Asynchronous Importance weighted Policy Optimization - an example of a principled off-policy learning algorithm which effectively mitigates training instability in large-scale training.
        \item \textbf{Fully Distributed Weight Synchronization via Direct GPU Memory Access:} We developed a fully distributed, GPU-native synchronization method to efficiently synchronize model weights between the offloaded generators and policy models. Leveraging direct GPU-to-GPU zero-copy transfers via NVLink, this method achieves linear scalability, enabling weight updates for terabyte-scale models across thousands of GPUs in approximately 2 seconds, as illustrated in Figure~\ref{fig:weights_update}.
    \end{itemize}
\end{itemize}
Further, to better convey important intuitions to practitioners, as to the source of the empirical speed-up, we build theoretical guarantees on why asynchronous RL training strictly speeds up over the synchronous counterpart (Section~\ref{sec:theory}). Finally, we finish with extensive experimental ablations along a few dimensions: the training time speed-up, the model quality as well the importance of off-policy corrections in asynchronous training (Section~\ref{sec:exp}). Taken together, we showcase the promise of \textsc{LlamaRL} to massively scale up RL training.

\section{Background}\label{sec:background}

We provide background on a few important topics related to this work.

\paragraph{\textbf{Large language models (LLM).}} LLMs are a powerful class of generative models. Upon large-scale pre-training and relatively lightweight fine-tuning, these models are incredibly versatile, capable of performing tasks like summarization, coding and translation, among others \citep{ziegler2019fine,Stiennon2020}. As a result, LLMs can be very helpful human assistants \citep{Ouyang2022RLHF,OpenAI2023GPT4,team2023gemini,Anthropic2023Claude,grattafiori2024llama}, and sometimes with specialized capabilities that exceed human users or even experts \citep{jaech2024openai,deepseek2025r1}. 

\paragraph{\textbf{Scaling laws.}} Pre-training research has shown that the performance of LLMs can be reasonably predicted by a power-law of model size, data size and compute \citep{Hestness2017Scaling, Kaplan2020Scaling,hoffmann2022training}, facilitating industry grade scaling of large models. Recently, it has been shown that scaling inference-time compute, elicited through RL on verifiable tasks, is an effective approach for empowering LLMs with profound reasoning capabilities \citep{zhang2024scaling,jaech2024openai,deepseek2025r1}, improving their mathematical problem solving, code generation and safety alignment \citep{jaech2024openai,deepseek2025r1,kimi2025,guan2024deliberative}. 
These recent advances underscore the potential of scaling up LLM compute, emphasizing the need to develop more efficient reinforcement learning systems.

\paragraph{\textbf{Reinforcement learning for LLMs.}}
Reinforcement learning has emerged as a key component of fine-tuning state-of-the-art LLMs, such as GPT-4 \citep{OpenAI2023GPT4}, Claude \citep{Anthropic2023Claude}, Gemini \citep{Google2023Bard,team2023gemini}, Llama 2-Chat \citep{Touvron2023Llama2}, and DeepSeek-R1 \citep{deepseek2025r1}, among others. RL allows LLMs to go beyond modeling the distribution of their supervised training data and adapt the sampling distribution to be more highly rated by reward functions. For Reinforcement Learning from Human Feedback (RLHF), where the goal is to make the model more creative, safe and helpful, such rewards are often provided by human raters \citep{Christiano2017,Ouyang2022RLHF,Bai2022Training,Bai2022ConstitutionalAI}. More generally, the rewards can also come from rule-based scorers, such as unit test execution for code applications \citep{li2022competition,gehring2024rlef,el2025competitive,wei2025swe} or matching against ground truth for mathematical reasoning \citep{uesato2022solving,lightman2023let}.

\paragraph{\textbf{Training algorithms and systems for RL.}} RL algorithms commonly used for training LLMs, such as Proximal Policy Optimization (PPO) \citep{Schulman2017PPO}, usually consist of four models: a policy model, a critic model (or value model), a reference policy model, and a reward model. PPO proceeds in iterations, each with three stages: (1) response generation using the policy model with a batch of prompts; (2) preparation of training data by scoring the generated responses through a single forward pass of the critic, reference policy, and reward models; (3) updating the actor and critic models through forward and backward computation. A number of RL pipelines \citep{ Christiano2017,ziegler2019fine,Stiennon2020,Ouyang2022RLHF,li2024remaxsimpleeffectiveefficient,ahmadian2024back,shao2024deepseekmath} follow similar stages but involve different numbers of models and data dependencies among the models. For example, the critic models can be optional and depending on the application, the reward models can be replaced by rule-based scorers.

Figure 1 depicts an example flow of online RL training only employing a policy model and several rule-based scorers for simplicity. We will use this RL training flow to explain our framework in the following sections. 

\begin{figure}[h!]
    \centering
    \includegraphics[width=0.75\linewidth]{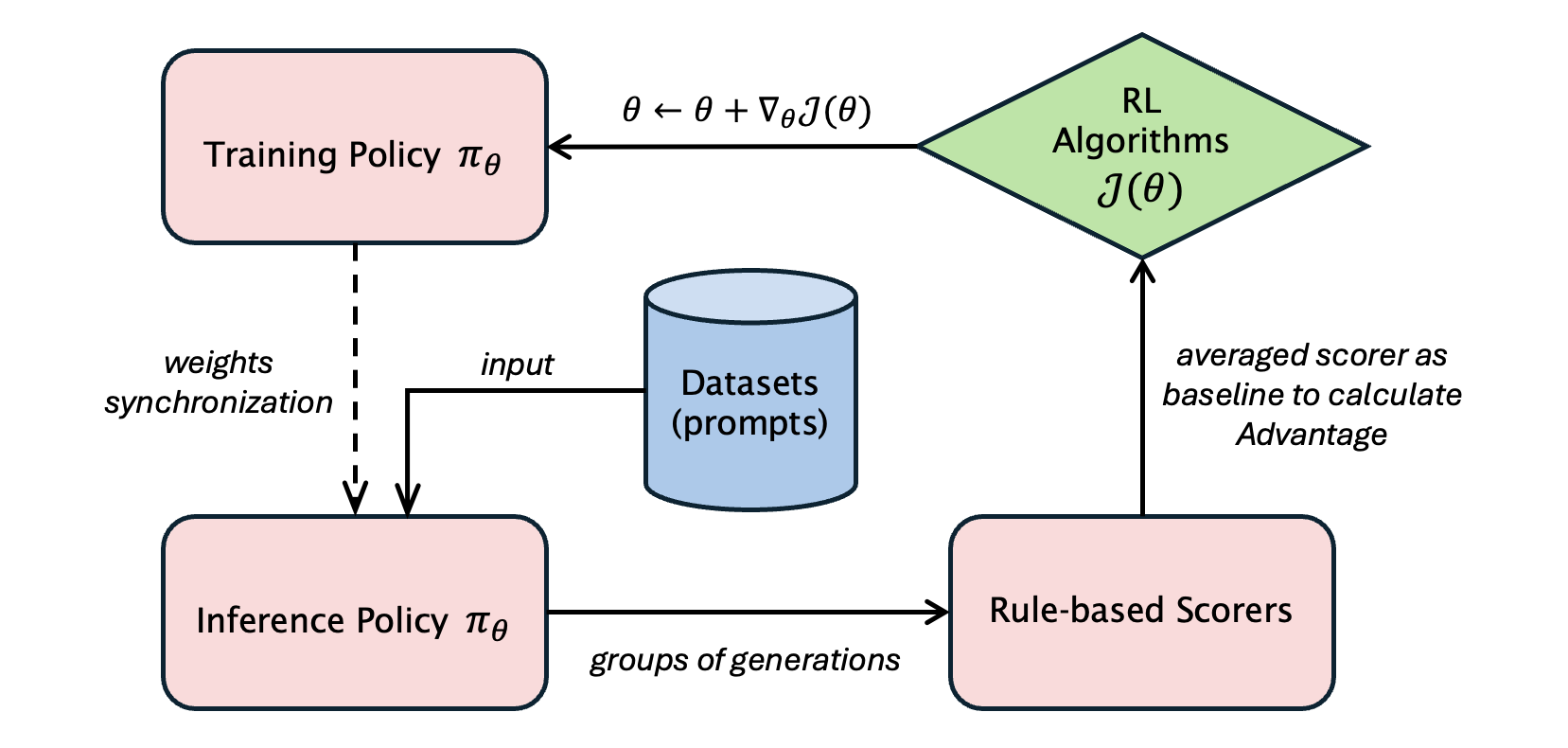}
    \caption{An example flow of online RL training. A reference model for regularization is dropped in this flow for simplicity. The example flow also foregoes a learned critic model, and instead estimates the baseline from group scores to calculate the advantage functions for policy update. The example makes use of rule-based scorers, as is often the case for code and reasoning applications. The policy model has two instances, implementing based on the Fully Sharded Data Parallel (FSDP) and CUDA Graph for training and inference optimizations, respectively.}
    \label{fig:example_rlhf_workflow}
\end{figure}

\section{Related Work} \label{sec:related}
Various frameworks have been proposed and implemented to support the RL training while enhancing the efficiency and scalability. We start with an overview of the existing framework design principles, see Table~\ref{tab:rlhf-frameworks} for a summary and comparison to \textsc{LlamaRL}.

\begin{table}[ht]
\centering
\begin{tabular}{c|c|c|c|c|c}
\toprule
\bfseries{Framework} & \begin{tabular}[c]{@{}c@{}}\bfseries{Model} \\ \bfseries{Placement} \end{tabular} & \begin{tabular}[c]{@{}c@{}}\bfseries{Weights} \\ \bfseries{Synchronization} \end{tabular} & \begin{tabular}[c]{@{}c@{}}\bfseries{Parallelism}   \\ \bfseries{(train/gen)} \end{tabular} & \bfseries{Async} & \begin{tabular}[c]{@{}c@{}}\bfseries{Model} \\ \bfseries{Scale} \end{tabular} \\ \midrule

\begin{tabular}[c]{@{}c@{}}DeepSpeed-Chat\\\smallcitep{yao2023deepspeedchateasyfastaffordable} \end{tabular} & Co-located & Communication & ZeRO/TP & No & 65B \\ \midrule

\begin{tabular}[c]{@{}c@{}}FlexRLHF\\\smallcitep{flexrlhf} \end{tabular} & Distributed & \begin{tabular}[c]{@{}c@{}}In-place update \\ (HybridEngine) \end{tabular} & ZeRO/TP & No & 66B \\ \midrule

\begin{tabular}[c]{@{}c@{}}NeMo-Aligner\\\smallcitep{shen2024nemoalignerscalabletoolkitefficient} \end{tabular} & \begin{tabular}[c]{@{}c@{}}Actor/Ref co-located \\ Critic/RM co-located \end{tabular} & \begin{tabular}[c]{@{}c@{}}In-place update \\ (TensorRT Refitter) \end{tabular} & 3D/3D & No & 405B \\ \midrule

\begin{tabular}[c]{@{}c@{}}OpenRLHF\\\smallcitep{hu2024openrlhfeasytousescalablehighperformance} \end{tabular} & Distributed & Communication & ZeRO/3D & No & 70B \\ \midrule

\begin{tabular}[c]{@{}c@{}}HybridFlow\\\smallcitep{sheng2024hybridflow} \end{tabular} & Distributed & \begin{tabular}[c]{@{}c@{}}In-place update \\ (3D-HybridEngine) \end{tabular} & \begin{tabular}[c]{@{}c@{}}(ZeRO, FSDP)\\/3D\end{tabular} & No & 70B \\ \midrule

\begin{tabular}[c]{@{}c@{}}AsyncRLHF\\\smallcitep{noukhovitch2024asynchronousrlhffasterefficient} \end{tabular} & \begin{tabular}[c]{@{}c@{}}Distributed \\ (single node) \end{tabular} & communication & \begin{tabular}[c]{@{}c@{}}(ZeRO, FSDP)\\/3D\end{tabular} & Yes & 8B \\ \midrule

\textsc{LlamaRL} & Distributed & communication & FSDP/3D & Yes & 405B \\ \midrule

\end{tabular}
\vspace{0.3em}
\caption{Comparison of RL training frameworks. Since all the frameworks in the table leverage different optimizations for training and inference, the column of weights synchronization refers to the method passing the weights from the training policy model to the inference policy model. The different parallelisms supported in the training and inference stages are listed in the table as well.}
\label{tab:rlhf-frameworks}
\end{table}

\paragraph{\textbf{Different Optimizations for Inference and Training.}}
Usually, a RL algorithm includes an inference phase to roll out the generations from the policy model, and a training phase for the policy model to learn the experience from the generations. Due to the different characters and constrains of the two phases, distinct optimization techniques should be applied to the inference and training phases, respectively. Therefore, most of the training frameworks set up an inference engine and a training engine of the policy model to execute in the corresponding phases. 

\paragraph{\textbf{Co-located Model Placement.}}
Including the inference policy and training policy models, the models in the RL algorithms can be allocated by mainly two strategies. Some earlier works adopts the co-location strategy, which places all models on each device and applies data and model parallelism techniques like ZeRO to parallelize the training \citep{rajbhandari2020zero}. An obvious drawback of this strategy is that, each model shares a limited memory to generate and process the data. Therefore, the training throughput becomes a bottleneck, especially for the generation phase. Nowadays, distributed strategy becomes a popular solution, which allocate each model to an individual processing group with an exclusive group of devices. 

There are optimization techniques to mitigate the memory limitation of the co-located strategy. \textit{DeepSpeed-Chat} \citep{yao2023deepspeedchateasyfastaffordable} implemented a Hybrid Engine to co-locate the inference policy and training policy. During the training phase, it utilizes ZeRO for efficient memory management, while switching the model weights to tensor parallelism for the generation phase. \textit{HybridFlow} \citep{sheng2024hybridflow} further developed a 3D-HybridEngine, allowing different 3D (data, tensor and pipeline) parallel configurations in the two stages and enabling zero memory redundancy and minimized communication overhead during the transition between two stages. \textit{NeMo-Aligner} \citep{shen2024nemoalignerscalabletoolkitefficient} co-located the models with the same architecture but different weights, for example policy model and reference model. Due to the sequential order of the execution, it first offloads the reference’s weights to CPU, and then swap them with the policy’s weights at the reference model inference step.

\paragraph{\textbf{Distributed Model Placement.}}
The distributed strategy requires the communications between the processing groups. \textit{FlexRLHF} \citep{flexrlhf} designed a pipeline manner process to overlap the computation and communication stages and minimize the communication overhead of the generations. However, the weight communication between training policy and inference policy is inevitable. In \textit{OpenRLHF} \citep{hu2024openrlhfeasytousescalablehighperformance}, the measured weight communication time demonstrates a faster-than-linear increases with the model size. A 70B model weights communication consumes 111 seconds. When the model size is scaled to the 405B, the weights communication time is estimated to be over 900 seconds based on the trends. The bottleneck is not on the GPU inter-node bandwidth, but on the expensive weights reloading operation. We will discuss more details on how to optimize the weights communication in the following sections.

\paragraph{\textbf{Asynchronous RL and Off-policyness}}
The asynchronous communication strategy between learner and actor have been a long standing practice to improve throughputs in large-scale RL. Asynchronization allows distributed models to execute in parallel without blocking each other, which fundamentally improves the training speed. It can lead to big boost in training efficiency, in particular for value-based RL (e.g., AlphaZero~\citep{alphazero} versus AlphaGoZero~\citep{alphagozero}, OpenGo~\citep{tian2019elf}), in which the issues of off-policyness may not be an issue. However, for policy-based RL training, which is a popular approach for LLM post-training, the issue of off-policyness can be more severe, leading to a fundamental question to explore. The algorithmic component of \textsc{LlamaRL} draws on a long history of off-policy learning, where importance sampling corrections have been applied in various forms \citep{precup2000eligibility,jie2010connection,wang2016sample,munos2016safe,Schulman2017PPO,espeholt2018impala}. Most closely related to our approach is IMPALA \citep{espeholt2018impala} where they apply importance correction to policy and critic learning. In \textsc{LlamaRL}, we have generalized this formulation more broadly for LLM fine-tuning, with combinations of different reward models, critics, etc. 

\paragraph{\textbf{Scaling up.}} Notably, non-LLM RL training with simulated environment often deal with very small models (on the order of millions of parameters) \citep{mnih2016asynchronous,espeholt2018impala,kapturowski2018recurrent}, and as a result, at training time the inference models are usually placed on CPU hosts. There are larger sized RL training examples AlphaStar~\citep{vinyals2019alphastar} and OpenAI Five~\citep{berner2019dota}, among others, but generally the model scales are still small by today's LLM standard. Scaling such ideas efficiently to the LLM scenarios where the model size has grown 100-1000x, poses major engineering challenges. Recently, \citep{noukhovitch2024asynchronousrlhffasterefficient} proposed an asynchronous RL training for 2.8B and 8B LLMs on multiple GPUs of a single host. 
In the context of large-scale RL, \citet{kimi2025} introduced the partial rollout method to enhance training efficiency during long-context generation. Their approach segments long responses into smaller portions across multiple iterations, while still maintaining a synchronous, sequential execution where the trainer and generator remain co-located.

\section{High Level Strategies}\label{sec:strategies}

We detail a number of key high level strategies that enable large-scale training with \textsc{LlamaRL}. For simplicity, we do not consider the reference model, and discuss algorithmic settings akin to the one showed in Figure~\ref{fig:example_rlhf_workflow}, where the critic model is removed and we apply rule-based scorers.

\subsection{Co-located model offloading} 
As mentioned before, the co-located model placement strategy limits the available memory usage for each model, and forces the models to share the same parallelisms. Therefore, our initial attempt is offloading certain modules with expensive memory or computation overhead to dedicated groups, namely distributed model placement. We offload the generation phase since it is a memory-bound process with major execution time contribution. The inference policy and training policy models are thus in different processing groups, implemented by Meta internal inference library and open-sourced \textit{FSDP}, respectively. Rule-based scorers are allocated with the training policy model, and computed with lightweight Python programs.

After offloading, the communications between processing groups are essential. There are mainly two types of communication. One is the communication of generations from the inference policy model, reward model or etc. Another is the communication of the model weights or logits. Instead of partial co-location of inference and training policy models such as Hybrid Engine which can save the weights communication overhead, we developed a distributed direct memory access (DDMA) method to achieve second-scale weights communications. The major reason of using the fully distributed model placement is because it provides a flexibility to adopt the asynchronous RL training.

\subsection{Asynchronous off-policy RL}
After switching to distributed groups to hold the models, the framework still suffers low GPU memory utilization and poor scalability due to the sequential execution order of the models in RL training. To overcome these issues, we apply an asynchronous off-policy learning algorithm for the \textsc{LlamaRL} framework, unblocking each component of the RL workflow. As comparison, Figure 2 demonstrates the process of (a)  synchronous on-policy RL and (b)  asynchronous off-policy RL.

For asynchronous RL, each processing group runs in parallel and communicates at the end of each training step. This design enables the trainer and generator components to compute concurrently, reducing overall training time. In contrast, the synchronous on-policy RL alternates between trainer and generator steps, blocking one another while the other component is running, which slows down the oevrall training. 

To further mitigate straggler effects—both within data-parallel groups and between trainer and generator components—we adopt a strategy inspired by partial rollouts~\citep{kimi2025}. Specifically, we break down long response generations, cache incomplete prompts, and resume them in subsequent iterations.

However, asynchronous learning inevitably introduces off-policyness: the policy model is trained on samples generated by a previous iteration of itself. Despite the off-policyness, our results suggest that with the help of off-policy corrections, the impact of off-policy training is negligible compared to fully on-policy training. We will discuss the algorithmic etails further in Section~\ref{sec:offloading_correctness}.
\begin{figure}[h!]
    \centering
    \includegraphics[width=1\linewidth]{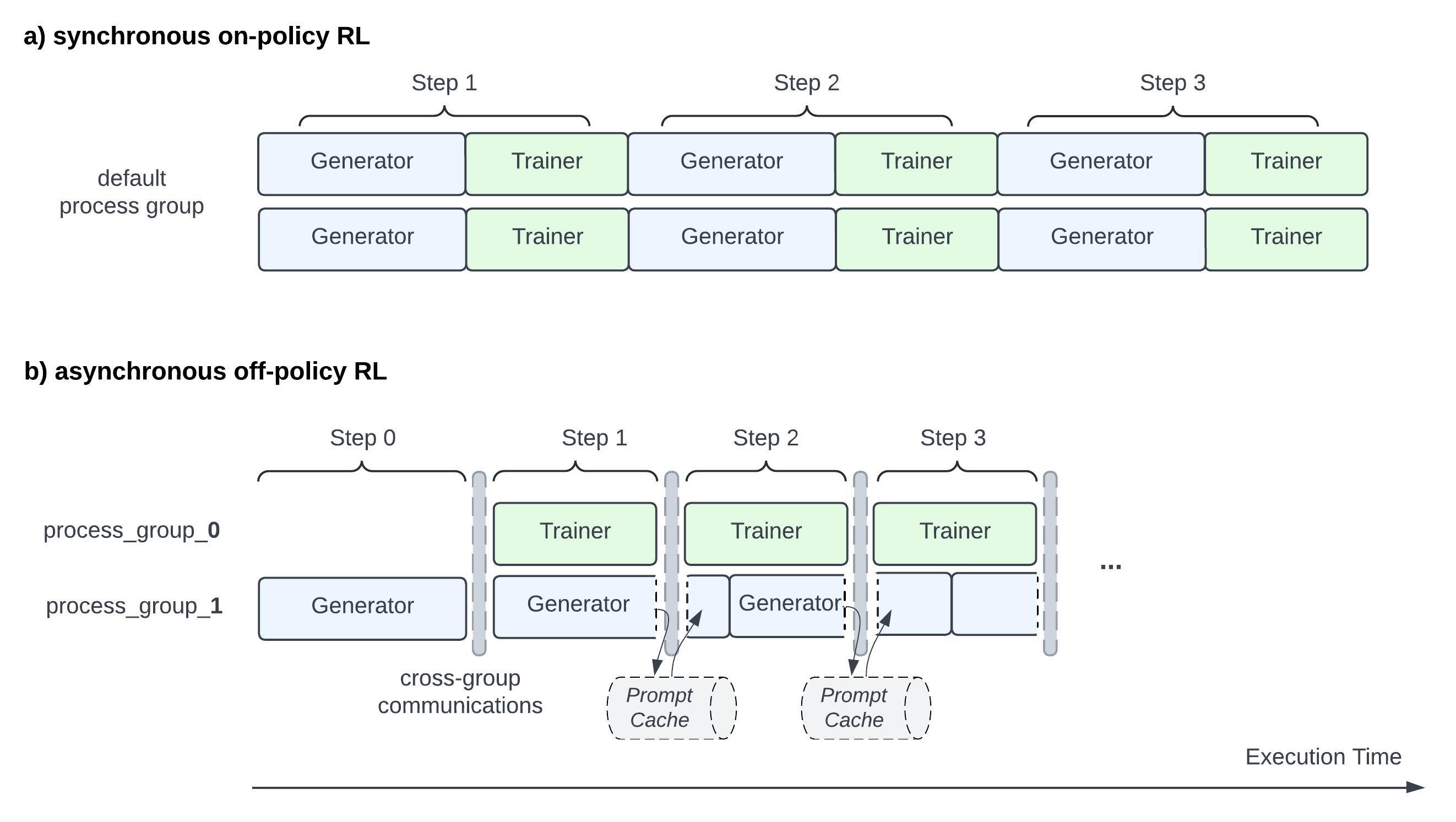}
    \caption{Process demonstrations of (a) synchronous on-policy RL, and (b) asynchronous off-policy RL. For asynchronous RL, the generator and trainer run in parallel without blocking one another, in contrast to synchronous RL. This design accelerates the overall training process significantly, without compromising model quality.}
    \label{fig:example_rlhf_workflow}
\end{figure}

\subsection{Fine-grained parallelisms and quantization}
From the example in Figure 2(b), GPU idle bubbles can be observed from the processing group 1 since the execution time is not perfectly matched among all the processing groups. To further optimize training efficiency, flexible and fine-grained parallelisms and quantization must be supported. In \textsc{LlamaRL}, models in different processing groups can use different parallelisms and data precision. We provide a few examples. 
\begin{itemize}
    \item \textbf{Data Parallelism} Back to the co-located stragety, all the models are stacking together, which means the dp size of the models are required to be the same. In \textsc{LlamaRL}, the dp size of the models are decoupled, which provides the flexibility to better align the throughput of different processing groups and improve the training efficiency.
    \item \textbf{Model Parallelism} Model parallelism is decouple among different processing groups. Based on the empirical knowledge, the smaller mp size (especially when mp > 8) in the inference side can significantly reduce the inter-node communications and thus reduce the generation time. 
    \item \textbf{Pipeline Parallelism} Pipeline parallelism distributes model computations across different pipeline stages. By decoupling the pp size among processing groups, it provides the flexible to fine tune the execution time of each processing group for better time-wise alignment.
    \item \textbf{Quantization} Quantization (\texttt{fp8} or \texttt{fp4}) on the inference side allows models to do generations with a smaller mp size, providing faster generation speed. 
\end{itemize}
The functionality is mainly implemented in the communication channels. More details can be found in the next Section.

\section{System Design}\label{sec:design}
In this section, we introduce the high level architecture and interfaces of \textsc{LlamaRL}. Unlike prior works such as OpenRLHF and HybridFlow that rely on third-party orchestration systems (e.g., Ray \citep{moritz2018ray}), \textsc{LlamaRL} is implemented purely on top of native PyTorch. This design choice reduces complexity, boosts reliability, and simplifies user workflows. Below, we explain the key architectural components, the interfaces that enable users to build RL algorithms, and an example of how a typical RL pipeline is assembled within \textsc{LlamaRL}.
\subsection{High Level Architecture}
\begin{figure}[h!]
    \centering
    \includegraphics[width=1\linewidth]{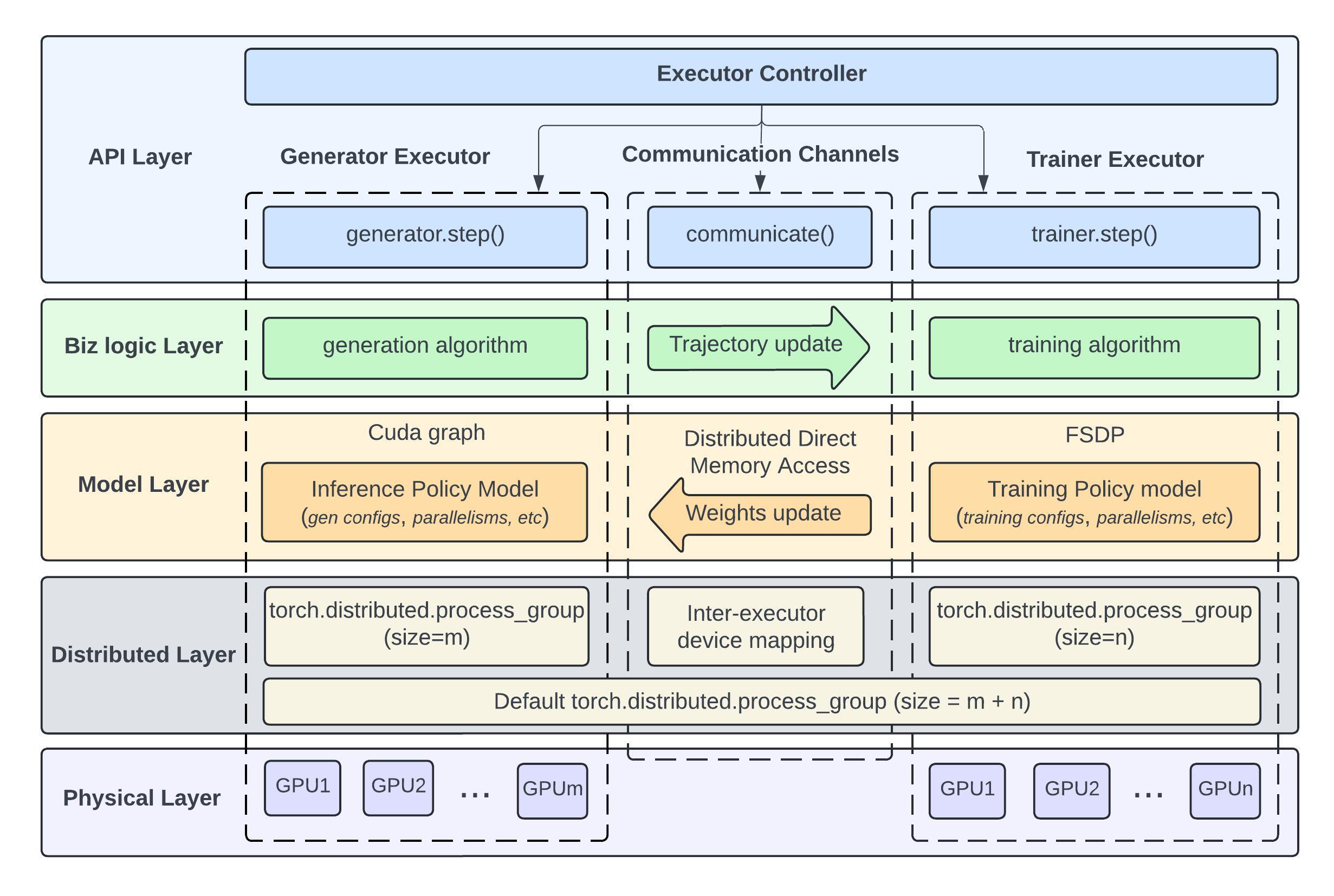}
    \caption{An example of \textsc{LlamaRL} architecture with a communication channel between two executors, managed by a single controller.}
    \label{fig:architecture}
\end{figure}
As shown in Figure~\ref{fig:architecture}, \textsc{LlamaRL} orchestrates several \textbf{executors}, each responsible for one or more stages of the RL pipeline such as policy inference, reward calculation and policy training. These executors are attached with distributed processing groups with their own GPU size and parallelism configurations. Since the nature of the distribution, executors rely on \textbf{communication channels} to pass data (e.g., prompts, generated trajectories, rewards) and updated model weights among them. Rather than relying on a heavyweight external scheduler, a single \textbf{controller} is responsible to launch, coordinate, and synchronize executors and data flows as needed.

\subsubsection{Executor}
RL algorithms such as PPO utilize multiple models. For the allocation of the multiple models, a straightforward strategy is stacking all the models across all the GPUs with the same parallelisms, namely co-located strategy. Due to the drawbacks discussed in the Section 2, a distributed model placement strategy is more preferred.

The \textsc{LlamaRL} uses the distributed strategy, and further encapsulates the distributed models in the executors. An executor is a self-contained unit that runs on a predefined number of GPUs with a specific parallelism configuration. A base class of an executor contains the following functions:

\begin{enumerate}
    \item \texttt{Init}: Load or construct a model, configure the parallel environment (e.g., FSDP, pipeline parallel, or tensor parallel groups), and prepare any needed data loaders or tokenizers.
    \item \texttt{step}: Executes one RL step (e.g., a forward/backward pass for policy training or a forward pass to generate samples).
    \item \texttt{save\_checkpoint}: Periodically saves the model state, optimizer state, and other progress markers.
    \item \texttt{get\_output / get\_model}: Exposes outputs or internal states (e.g., recently generated sequences) so that other executors can receive them via a communication channel.
\end{enumerate}

Different executor classes are further implemented inheriting the base executor class. For example,  figure~\ref{fig:architecture} demonstrates two executors, generator and trainer for the RLOO algorithm, working for inference and training phases respectively. Rule-based reward models are included in the trainer. By isolating RL stages into different executors, \textsc{LlamaRL} makes it easier to scale up or swap out individual components---such as changing the model parallelism or switching from PPO to another RL variant---without monolithic code changes.

\subsubsection{Communication Channel}
A communication channel defines a directed link of communication paradigm between executors. It includes:

\begin{itemize}
    \item A \texttt{name}, for referencing data or weights being transferred.
    \item An \texttt{outbound\_executor} (the sender) and an \texttt{inbound\_executor} (the receiver).
    \item A \texttt{communication\_type}, indicating how data is distributed:
    \begin{itemize}
        \item \textbf{BROADCAST}: Outbound data is sent identically to each inbound process.
        \item \textbf{SCATTER}: Outbound data is partitioned and delivered in chunks to different inbound processes.
        \item \textbf{GATHER}: Data from multiple outbound processes is aggregated by a single inbound executor.
    \end{itemize}
\end{itemize}

The complicated interactions between executors can therefore be defined via a list of communication channels. Crucially, they also handle \texttt{model weight updates}, which is a central challenge in large-scale RL.

\subsubsection{Controller}
Finally, a \textbf{Controller} ties executors and channels together into a single training process, as shown in the Algorithm~\ref{alg:executor_controller}. Under the hood, it initializes the distributed environment, launches each executor on the correct GPU ranks, and orchestrates each training ``tick'' by:

\begin{enumerate}
    \item \texttt{Setting the iteration step} for each executor (so each knows which mini-batch or prompts to process).
    \item \texttt{Executing communication channels} in the correct order. For example, the Generator executor's outputs are passed to the Reward executor, then the Reward executor's outputs flow to the Policy Trainer.
    \item \texttt{Triggering each executor's step}. In a single iteration, multiple executors run asynchronously, each performing its own computations without waiting for others to finish.
    \item \texttt{Saving checkpoints}. Periodically, each executor writes out model or state data independently (or under user-defined triggers).
\end{enumerate}

Because each executor is an autonomous SPMD process, the Controller remains concise and easy to reason about---essentially just an event loop.

\begin{algorithm}[h]
\caption{ExecutorController Training Loop Overview}
\label{alg:executor_controller}

\KwIn{
  \texttt{executor\_group}: \texttt{List[Executor]},
  \texttt{communication\_channels}: \texttt{List[CommunicationChannel]}, \\
  \texttt{max\_steps}: \texttt{int},
  \texttt{init\_communication\_channels}: \texttt{Optional[List[CommunicationChannel]]}
}
\KwOut{Control the training by coordinating executors}

\textbf{Class:} \texttt{ExecutorController} \; \
\Indp
    \texttt{executor\_context}: \texttt{ExecutorContext} $\leftarrow$ Init a context object which stores all the torch distributed groups across different different processes in different executors. These information will be used by communication channels \; \
    \texttt{local\_executor}: \texttt{Executor} $\leftarrow$ The executor instance local to the current process. Each process runs the same controller logic independently (SPMD model). \;
\Indm
\textbf{Function:} \texttt{init()} \; \
\Indp
    Set \texttt{torch.distributed.group.WORLD} to local executor group\; \
    Call \texttt{local\_executor.init()}\; \
\Indm
\textbf{Function:} \texttt{run()} \; \
\Indp
    \While{\texttt{local\_executor.curr\_step} $<$ \texttt{max\_steps}}{ \
        Call \texttt{local\_executor.set\_step()}\; \
        Set \texttt{torch.distributed.group.WORLD} to global group\; \
        \ForEach{channel in communication channels}{
            Call \texttt{communicate(channel)}\;
        }
        Set \texttt{torch.distributed.group.WORLD} to local group\; \
        Call \texttt{local\_executor.step()}\; \
        Call \texttt{local\_executor.save\_checkpoint()}\;
    }
    Call \texttt{executor\_context.post\_training\_step()} \; \
    Call \texttt{executor\_context.shutdown()}\; \
\Indm
\textbf{Function:} \texttt{communicate(channel)} \; \
\Indp
    \If{\texttt{local\_executor == channel.outbound\_executor}}{
        Call \texttt{SEND\_OPS[channel.communication\_type](local\_executor, channel)}\;
    } \If{\texttt{local\_executor == channel.inbound\_executor}}{
        Call \texttt{RECV\_OPS[channel.communication\_type](local\_executor, channel)}\;
    }
\Indm
\end{algorithm}
\subsection{Distributed Direct Memory Access (DDMA) Weights Update}
\begin{figure}[h!]
    \centering
    \includegraphics[width=0.7\linewidth, height=0.4\textheight]{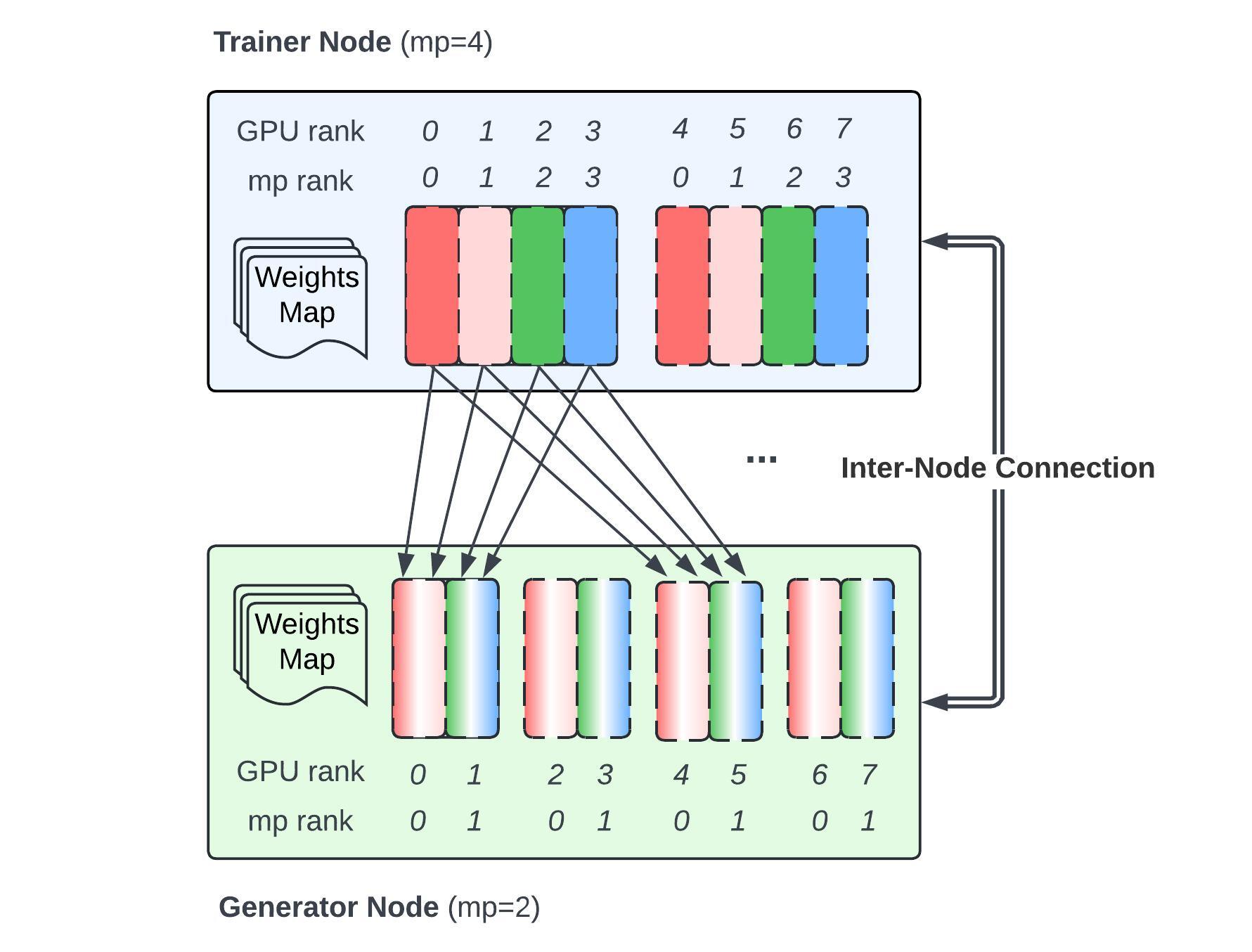}
    \caption{Model weights synchronization via distributed direct memory access (DDMA).}
    \label{fig:weights_update}
\end{figure}
One bottleneck for large-scale RL is the step-wise synchronization of model weights between executors, i.e., sending updated policy parameters from a training executor to an inference executor. Conventional methods are usually performed in a parameter server (PS)~\citep{li2014scaling_ps}, which functions as a key-value store for the current model parameters. However, the traditional parameter server architecture presents two major bottlenecks: (1) insufficient network bandwidth—especially when the PS is co-located with the worker nodes~\citep{chilimbi2014adam} and (2) an inefficient software stack that cannot scale effectively \citep{potluri2013gpudirect}. In particular, data copying, aggregation, and optimization overheads in the PS software stack hinder throughput, even on high-bandwidth networks.

We propose a \textbf{Distributed Direct Memory Access (DDMA)} method, which is dedicated for efficient weights synchronization in large-scale RL training with language models. As illustrated in figure~\ref{fig:weights_update}, DDMA introduces key optimizations that target the aforementioned bottlenecks:
\begin{enumerate}
\item \textbf{Fully Distributed:} Each GPU only stores or updates its assigned shards, leveraging the same tensor and parallel groups used during training. This eliminates memory bottlenecks on any single node and enables the linear scalability in both memory and network throughput up to thousands of GPU cards and several terabytes model weights with about 2 seconds weights update delay.
\item \textbf{Direct Memory Access for Weights Update:} By leveraging NVIDIA's intra-node NVLink and inter-node Infiniband interconnects, DDMA supports zero-copy communication by allowing direct memory transfers between CUDA memory regions across devices—bypassing CPU memory altogether. This dramatically reduces latency in the weights update path and alleviates bottlenecks caused by data movement that can inflate step times by hundreds of seconds at large scales~\citep{hu2024openrlhfeasytousescalablehighperformance}.
\end{enumerate}

\subsection{Example: RL with Offloaded Reward and Generator}
\begin{algorithm}[H]
\caption{Main Entry Point for a RL algorithm with reward and generator offloaded}
\label{alg:main_entrypoint}
\KwIn{
  \texttt{params}: \texttt{PPOParams},
  \texttt{num\_generator\_gpus}: \texttt{int},
  \texttt{num\_reward\_gpus}: \texttt{int}
}
\KwOut{Initialize distributed components and run training loop}

\textbf{Function:} \texttt{main(params, num\_generator\_gpus, num\_reward\_gpus)} \;
\Indp
    Call \texttt{init\_torch\_distributed} \;
    
    \texttt{num\_all\_gpus} $\leftarrow$ \texttt{torch.distributed.get\_world\_size()}\;
    \texttt{num\_trainer\_gpus} $\leftarrow$ \texttt{num\_all\_gpus - num\_generator\_gpus - num\_reward\_gpus}\;

    Init \texttt{base\_generator: BaseGenerator(params, num\_generator\_gpus)} \CommentSty{extends \texttt{Executor}} \;
    Init \texttt{ppo\_trainer: PPOTrainer(params, num\_trainer\_gpus)} \CommentSty{extends \texttt{Executor}} \;
    Init \texttt{reward\_calculator: RewardCalculator(params, num\_reward\_gpus)} \CommentSty{extends \texttt{Executor}} \;

    Init \texttt{executor\_controller: ExecutorController( \\
        \quad executor\_group=[base\_generator, reward\_calculator, ppo\_trainer], \\
        \quad communication\_channels=[ \\
        \qquad WeightsCommunicationChannel(\\
        \qquad "policy\_model", ppo\_trainer $\rightarrow$ base\_generator, DDMA\_WEIGHTS\_UPDATE), \\
        \qquad CommunicationChannel(\\ 
        \qquad "completions", base\_generator $\rightarrow$ reward\_calculator, GATHER), \\
        \qquad CommunicationChannel(\\
        \qquad "completions\_with\_reward", reward\_calculator $\rightarrow$ ppo\_trainer, SCATTER)
        ], \\
        \quad max\_steps=params.max\_steps \\
    )}\;
    Call \texttt{executor\_controller.run()}\;
\Indm
\end{algorithm}

Algorithm~\ref{alg:main_entrypoint} demonstrates how you might assemble a \texttt{policy trainer}, a \texttt{reward calculator}, and a \texttt{generator} into a single \textsc{LlamaRL} job. Key points include:

\begin{itemize}
    \item \textbf{Executor Initialization.} Each executor is allocated a distinct group of GPU processes (lines~4--7 in the pseudocode).
    \item \textbf{Communication Channels.} We define channels for passing updated policy weights to the generator, generated sequences to the reward model, and reward-scored sequences back to the trainer (lines~10--16).
    \item \textbf{Offloading.} The reward model and the generator run on different sets of GPUs so that each can specialize its parallelism strategy.
    \item \textbf{Controller.} Once set up, the \texttt{ExecutorController.run()} method handles the asynchronous scheduling and repeated iteration (lines~18--19).
\end{itemize}

This modular design allows you to swap in alternative reward models, add additional models (e.g., critics, reference policies), or implement a different RL optimization algorithm without altering the rest of the pipeline.

By separating each component of the RL pipeline into its own \texttt{Executor} and using \texttt{CommunicationChannels} for data exchange, \textsc{LlamaRL} provides a highly scalable yet flexible framework. The distributed approach coupled with \texttt{DDMA} greatly alleviates GPU-memory pressure and reduces idle time in the RL pipeline. In the following sections, we evaluate how these design choices translate into real-world speedups and scaling performance.

\section{\textsc{LlamaRL} Off-policy Learning}
\label{sec:offloading_correctness}

In \textsc{LlamaRL}, we adopt an importance sampling based off-policy learning framework to the large scale asynchronous training process. To set up notations, we let $x\in \mathcal{X}$ be the set of prompts for RL training and $y\in \mathcal{Y}$ be possible generations. The generation $y=y_{1:T}$ is made up of individual tokens $(y_t)_{t=1}^T$ where $T$ denotes the generation length. Below, we elaborate on an example using the policy gradient algorithm.

Let the learner model be $\pi$, which defines distributions over generations $y=y_{1:T}$ in an auto-regressive manner $\pi\left(y\;|\;x\right) = \Pi_{t=1}^T \pi\left(y_t\;|\;x,y_{1:t-1}\right)$. The generations are sampled by an actor model, defined through another network in an auto-regressive manner through inference $y_t\sim \mu\left(\cdot\;|\;x,y_{1:t-1}\right)$. Due to asynchronous training, $\mu$ differs from $\pi$ (e.g., lagging by one update step) and hence the two networks produce different distributions. The importance sampling ratio 
$\frac{ \pi\left(y_t\;|\;x,y_{1:t-1}\right)}{\mu\left(y_t\;|\;x,y_{1:t-1}\right)}$ adjusts for the off-policy discrepancy between the two networks during learning.

During training, generation $y_t$ along with the probability $\mu\left(y_t\;|\;x,y_{1:t-1}\right)$ are communicated from the generator to the trainer. Our final update to the learner network is an importance sampling weighted policy gradient
\begin{align*}
   \sum_{t=1}^T \min\left(\frac{ \pi\left(y_t\;|\;x,y_{1:t-1}\right)}{\mu\left(y_t\;|\;x,y_{1:t-1}\right)}, \rho\right) \cdot A(x,y_{1:t})\nabla \log \pi\left(y_t\;|\;x,y_{1:t-1}\right),
\end{align*}
where $A(x,y_{1:t})= r(x,y_{1;t})-v(x,y_{1:t})$ is the per-token advantage estimate. Here, $r(x,y_{1;t})$ is the per-token reward which can be computed from a reward model \citep{Christiano2017,Ouyang2022RLHF}. For many applications, the reward can be computed with rule-based scorers \citep{gehring2024rlef,lightman2023let}, resulting in a single reward per generation $r(x,y)$. We will focus on this case in the ensuing sections. The reward is often added with the KL regularization $\lambda_\text{KL}\cdot D_\text{KL}\left(\pi(y|x),\pi_\text{base}(y|x)\right)$ adjusted by parameter $\lambda_{\text{KL}}$, to penalize deviation from the reference policy $\pi_\text{base}$. 

The baseline $v(x,y_{1:t})$ is designed for variance reduction, and there are a few ways to compute it. For example, it can be computed from a critic model \citep{Schulman2017PPO} or using branched rollout \citep{schulman2015trust,kazemnejad2024vineppo}. Another simple yet effective approach, which we will focus on below and compatible with the workflow in Figure~\ref{fig:example_rlhf_workflow}, is that from a single prompt $x$, we sample $n$ generations and compute the baseline as the mean rewards from these generations $y_i,i\in[n]$ with reward $r(x,y_i)$. This produces 
a constant baseline for the whole sequence of tokens $
    v(x,y_{1:t}) = \frac{1}{n} \sum_{i=1}^n r_i(x,y) $
which has proved reasonably effective for a number of applications \citep{ahmadian2024back,shao2024deepseekmath}.

Noticeably, we apply a fixed constant $\rho$ to clip the importance sampling ratio $\frac{ \pi\left(y_t\;|\;x,y_{1:t-1}\right)}{\mu\left(y_t\;|\;x,y_{1:t-1}\right)}$ from above. This clipping strategy introduces a bias and variance trade-off in the gradient estimation: when $\rho$ is assigned a large value, the estimation has less bias but more variance due to the importance sampling weights. We call the overall algorithm \texttt{AIPO}
(Asynchronous Importance weighted Policy Optimization).

\paragraph{\textbf{Mitigating training instability.}} In practice, we find clipping constant of $\rho\in[2,10]$ seem to work generally well and help relieve much of the training instability arising in large-scale experiments. Such training instability is usually manifested as sudden or slow drops in training performance, and tends to arise more often due to sophisticated interactions between  asynchronous training and other factors such as model architecture, data mixtures, multiple modalities, additional training and inference optimization, etc.

\paragraph{\textbf{Comparison to PPO clipping.}} For readers familiar with the topic, it is clear that the above correction differs slightly from the double-sided clipping in PPO \citep{Schulman2017PPO} and more recently GRPO \citep{shao2024deepseekmath}. We have early ablations showing that the single sided clipping works better in our setting. We also argue that there are good reasons for the different algorithmic designs: in a nutshell, the PPO clipping is based on the trust region policy optimization formulation, which differs subtly from the asynchronous training setting in \textsc{LlamaRL}. We provide additional discussions in Appendix~\ref{apepndix:off-policy-comparative}.

\section{Theoretical justification of \textsc{LlamaRL} Speedup} \label{sec:theory}
In this section, we aim to explain the efficiency gains achieved by \textsc{LlamaRL} from a theoretical perspective. We hope this sheds light on where the speed-up comes from, and helps practitioners build a good mental models for such training systems in general. 

More precisely, we will prove the following theorem, stated informally below for ease of understanding:

\begin{theorem}[\textbf{Theoretical speed up of LlamaRL}, informal version of Theorem~\ref{main_theorem}] 
\label{informal}
Given the same hardware budget and memory constraints, \textsc{LlamaRL} can be configured to perform reinforcement learning (RL) strictly faster than any possible configuration of a traditional synchronous RL framework. 
\end{theorem}

At its core, the proof of Theorem \ref{informal} frames the search for optimal efficiency as a constrained optimization problem governed by GPU memory. We shall mathematically formulate both the synchronous and asynchronous frameworks and derive closed-form solutions for each. The key insight is that \textsc{LlamaRL}'s asynchronous design decomposes the shared memory constraint of the synchronous baseline into two independent constraints—one for the trainer and one for the generator. This removes the mutual dependency between the two components, allowing each to be optimized separately and hence yielding larger admissible region.

We start by introducing a few fundamental quantities based on structure of a typical RL training framework. Then, under reasonable assumptions that generally hold with modern hardware, we present a formal proof that our asynchronous design yields strictly higher efficiency than its traditional synchronous counterpart. Below, we begin by introducing a set of universal constants and notations that will be used throughout our derivations.

\begin{definition}
\label{consts} (\textbf{Universal constants of GPU})
Let $G_0$ denote the total number of available GPUs, $B_0$ the global batch size, $M_0$ the maximum available memory per GPU, and $W_0$ the memory footprint of a single model replica. Let $b_t$ represent the microbatch sizes used during training and $b_g$ represent the generator batch size which indicates the max decoding concurrency. Finally, let $m_t$ and $m_g$ denote the model parallel degrees (i.e., number of GPUs per model instance) for the trainer and generator.
\end{definition}

In Table~\ref{mem}, we list major model components (model weights, optimizer states, gradients, hidden activations, KV cache) that lead to primary consumption of per-GPU memory in both the trainer and generator. As an example, with model parallelism degree $m_t$, the memory consumption of a single model with size $W_0$ leads to a per-GPU memory consumption of $W_0/m_t$.

\begin{table}[ht]
\centering
\begin{tabular}{c|ccc}
\toprule
\bfseries{Name}                                                       & \bfseries{Location}  & \begin{tabular}[c]{@{}c@{}}\bfseries{Size per}   \\ \bfseries{GPU} \end{tabular} & \bfseries{Free-able} \\ \midrule
\begin{tabular}[c]{@{}c@{}}\emph{Trainer} \\ \emph{model}\end{tabular}   & Trainer   & $W_0/m_t$                                               & No        \\ \midrule
\emph{Adam states}                                                & Trainer   & $2W_0/m_t$                                              & No        \\ \midrule
\emph{Gradients}                                                  & Trainer   & $W_0/m_t$                                               & Yes       \\ \midrule
\emph{Activations}                                                & Trainer   & $A_t\cdot b_t / m_t$                                    & Yes       \\ \midrule
\begin{tabular}[c]{@{}c@{}}\emph{Generator} \\ \emph{model}\end{tabular} & Generator & $W_0 / m_g$                                             & No        \\ \midrule
\emph{KV cache}                                                   & Generator & $K_g\cdot b_g / m_g$                                    & No      \\ \bottomrule
\end{tabular}
\vspace{0.3em}
\caption{CUDA memory footprint of major components in a typical RL training iteration.  Here, $A_t$ and $K_g$ are constants determined solely by model size and are independent of batch size or model parallel sizes.}
\label{mem}
\end{table}

We introduce a few useful concepts that will simplify subsequent presentation and analysis.

\begin{definition}[\textbf{Processing time}]
\label{mini_batch_time}
For a given setup, let $\tau_t(b)$ and $\tau_g(b)$ denote the training and generation time for a batch of size $b$, respectively. Further, let us define the per-sample processing times as:
\begin{eqnarray}
\begin{split}
\eta_t(b) &\coloneqq \frac{\tau_t(b)}{b},
\eta_g(b) &\coloneqq \frac{\tau_g(b)}{b}.
\end{split}
\end{eqnarray}
\end{definition}

Note that the processing time is inversely proportional to throughput, i.e., smaller processing time means larger throughput. We now introduce the following key assumption underlying the main result:

\begin{assumption}[\textbf{Batch size scaling}]
\label{ass}
Both $\eta_t$ and $\eta_g$ are monotonically decreasing functions of batch size $b$. 
\end{assumption}
This assumption is well-supported by empirical observations in large-scale deep learning on modern GPUs. As batch size increases, per-sample processing time tends to decrease due to improved hardware utilization—such as higher arithmetic intensity, more efficient memory access patterns, and better kernel scheduling. On architectures like NVIDIA’s A100 and H100, larger batches better amortize compute and memory overhead, particularly for GEMM operations, resulting in sub-linear growth in total compute time and thus a decreasing trend in per-sample cost. 

This phenomenon has been widely observed and reported in prior work on efficient training and inference, such as DeepSpeed \citet{Rajbhandari2019ZeROMO}, where they leverage larger batch sizes to maximize GPU utilization and improve scalability. We empirically validated this assumption across a wide range of batch sizes in our setup, as shown in Figure~\ref{mono}.

\begin{figure}[h]
\centering
\includegraphics[width=\textwidth, height=6.25cm]{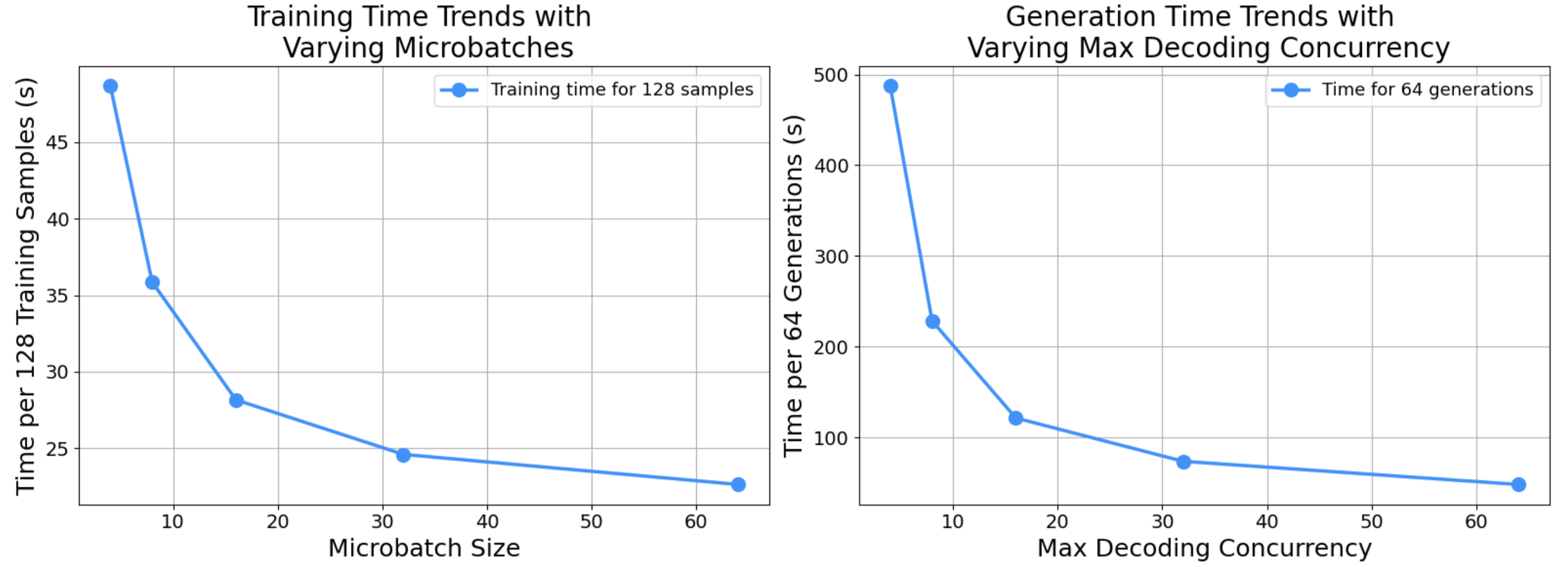}
\caption{\small{Empirical verification for Assumption~\ref{ass} on batch size scaling with the 70B model. \textbf{Left:} Training time per 128 samples decreases with increasing microbatch size. \textbf{Right:} Generation time per 64 completions decreases with increasing maximum decoding concurrency. Both results illustrate sub-linear growth in total processing time, supporting the assumption that per-sample time ($\eta_t$, $\eta_g$) decreases with batch size.}}
\label{mono}
\end{figure}

Now, we connect the above scaling discussion back to RL training speed-up,
by formally defining the notion of the RL step time for a given RL training framework:
\begin{definition}[\textbf{RL step time}]
Given RL framework $\mathscr{F}$, we define the step time $T_{\mathscr{F}}$ as a function of the microbatch sizes $(b_t, b_g)$ and model parallel sizes $(m_t, m_g)$. In particular, we consider:
\begin{itemize}
    \item $T_{\textsc{LlamaRL}}(b_t, b_g, m_t, m_g, \theta)$, which allows independent \texttt{mp\_size} trainer and generator as well as flexible GPU allocations. Here $\theta \in (0, 1)$ denotes the fraction of GPUs allocated to the trainer, with the remaining $(1 - \theta)\cdot G_0$ used by the generator.  
    \item $T_{\text{baseline}}(b_t, b_g, m)$, corresponding to the synchronous baseline, where the constraint $m_t = m_g \equiv m$ holds.
\end{itemize}
\end{definition}
Under mild assumptions—specifically, that inter-device communication overhead is negligible (as supported by our empirical measurements), and that the components listed in Table~\ref{mem} account for the dominant sources of CUDA memory consumption during RL training—we now establish a formal efficiency gain of \textsc{LlamaRL} over the synchronous baseline.

\begin{theorem}[\textbf{Theoretical speed-up of \textsc{LlamaRL}}]
\label{main_theorem}
Under the setup in Definition~\ref{consts} and Assumption \ref{ass}, there exist $b_t$, $b_g$, $m_t$, and $m_g$ satisfying the GPU memory constraint such that $T_{\textsc{LlamaRL}}(b_t, b_g, m_t, m_g)$ is strictly smaller than the minimum possible $T_{\text{baseline}}$.
\end{theorem}

\begin{remark}
It is worth noting that the theorem is established under minimal assumptions. In particular, we do not assume any specific analytical form for the batch processing time $\tau$ or the per-sample processing time $\eta$ (see Definition \ref{mini_batch_time}). The only requirement is that $\eta$ decreases monotonically in batch size—a behavior that is both widely reported by prior work in practice and empirically validated in our experiments, as we have discussed earlier.
\end{remark}

\begin{proof}

We begin by computing the step time for each framework. In the synchronous case, the total step time is the sum of the generator and trainer times. In contrast, for the asynchronous framework, the step time is determined by the slower of the two. Using this observation along with Definition~\ref{mini_batch_time}, we can express both step times as follows: 

\begin{eqnarray}
\begin{split}
\label{baseline_steptime}
T_{\text{baseline}}(b_t, b_g, m) &= \frac{B_0m}{G_0b_t}\cdot \tau_t(b_t) + \frac{B_0m}{G_0b_g}\cdot\tau_g(b_g) \\
&= \frac{B_0}{G_0}\cdot m \cdot {(}\eta_t + \eta_g{)}, 
\end{split}
\end{eqnarray}
\begin{eqnarray}
\begin{split}
\label{llamarl_steptime}
T_{\textsc{LlamaRL}}(b_t, b_g, m_t, m_g, \theta) &= \max\left(\frac{B_0m_t}{\theta G_0b_t}\cdot \tau_t, \frac{B_0m_g}{(1-\theta) G_0b_g}\cdot \tau_g\right) \\
&= \frac{B_0}{G_0}\cdot\max\left(\frac{\eta_tm_t}{\theta}, \frac{\eta_gm_g}{1-\theta}\right).
\end{split}
\end{eqnarray}

On the other hand, by Table \ref{mem}, the memory footprint per GPU of the trainer is given by
\begin{eqnarray}
\label{train_foot}
\frac{4W_0}{m_t} + \frac{A_tb_t}{m_t}. 
\end{eqnarray}
Similarly, the generator consumes
\begin{eqnarray}
\label{gen_foot}
\frac{W_0}{m_g} + \frac{K_gb_g}{m_g}. 
\end{eqnarray}

Now combining Eqn. \eqref{baseline_steptime}, \eqref{llamarl_steptime}, \eqref{train_foot} and \eqref{gen_foot}, the admissible minimizer of $T_{\text{baseline}}$ is a solution to the following constrained optimization problem: 

\begin{eqnarray}
\begin{split}
\label{baseline_prob}
&\min_{b_t, b_g, m}\quad \frac{B_0}{G_0}\cdot m \cdot {(}\eta_t(b_t) + \eta_g(b_g){)}\\
&\text{subject to}\quad \frac{(4W_0 + A_tb_t) + (W_0 + K_gb_g)}{m} \le M_0.
\end{split}
\end{eqnarray}

Likewise, using Eqn. \eqref{llamarl_steptime}, \eqref{train_foot} and \eqref{gen_foot}, the admissible minimizer of $T_{\textsc{LlamaRL}}$ solves \footnote{To simplify our argument, we omit the constraint $m \le G_0$. But note this condition is always satisfied by the optimal solutions to \eqref{baseline_prob} and \eqref{llamarl_prob} under sufficient GPU budget $G_0$, which holds in all practical settings we consider.}
\begin{eqnarray}
\begin{split}
\label{llamarl_prob}
&\min_{b_t, b_g, m_t, m_g}\quad \frac{B_0}{G_0}\cdot\max\left(\frac{\eta_t(b_t)m_t}{\theta}, \frac{\eta_g(b_g)m_g}{1-\theta}\right)\\
&\text{subject to}\quad 
\left\{
\begin{array}{l}
\displaystyle \frac{4W_0 + A_tb_t}{m_t}\le M_0, \\
\\
\displaystyle \frac{W_0 + K_gb_g}{m_g} \le M_0.
\end{array}
\right.
\end{split}
\end{eqnarray}

Now, consider a minimizer to problem \eqref{baseline_prob}, again denoted by $(b_t^*, b_g^*, m^*)$, achieving minimium admissible step time $T_{\text{baseline}}^*$. By Lemma \ref{constraint_eq}, $(b_t^*, b_g^*, m^*)$ satisfies \eqref{constraint_eq_id}, thus the following strict inequalities hold
\begin{eqnarray}
\begin{split}
\label{strict_ineq}
\left\{
\begin{array}{l}
\displaystyle \frac{4W_0 + A_tb_t^*}{m^*} < M_0, \\
\\
\displaystyle \frac{W_0 + K_gb_g^*}{m^*} < M_0.
\end{array}
\right.
\end{split}
\end{eqnarray}
Conversely, consider a solution to problem \eqref{llamarl_prob}, denoted by $(b_t^{**}, b_g^{**}, m_t^{**}, m_g^{**}, \theta^{**})$. By Lemma \ref{constraint_eq2}, it must satisfy the following constraints
\begin{eqnarray}
\begin{split}
\label{optim_m}
\left\{
\begin{array}{l}
\displaystyle m_t^{**} = \frac{(4W_0 + A_tb_t^{**})}{M_0}, \\
\\
\displaystyle m_g^{**} = \frac{(W_0 + K_gb_g^{**})}{M_0}.
\end{array}
\right.
\end{split}
\end{eqnarray}
Next, define the variables $T_t^{**}$ and $T_g^{**}$ for the optimal trainer and generator time 
\begin{eqnarray}
\begin{split}
\label{decoupled_sols}
\left\{
\begin{array}{l}
\displaystyle T^{**}_t \eqqcolon \min_{b_t}\frac{(4W_0 + A_tb_t)\eta_t}{M_0},\\
\\
\displaystyle T^{**}_g \eqqcolon \min_{b_g}\frac{(W_0 + K_gb_g)\eta_g}{M_0}.
\end{array}
\right.
\end{split}
\end{eqnarray}

Then we obtain 
\begin{eqnarray}
\begin{split}
\label{last}
T_{\text{baseline}}^* &= \frac{B_0}{G_0}\cdot \frac{(4W_0 + A_tb_t^*) + (W_0 + K_gb_g^*)}{M_0}\cdot (\eta_t^* + \eta_g^*) \\
&>_{(a)} \frac{B_0}{G_0} \cdot \bigg{(}\frac{(4W_0 + A_tb_t^*)\eta_t^*}{M_0} + \frac{(W_0 + K_gb_g^*)\eta_g^*}{M_0}\bigg{)}\\
&>_{(b)} \frac{B_0}{G_0} \cdot \bigg{(} T^{**}_t + T^{**}_g\bigg{)}\\
& = \frac{B_0}{G_0} \cdot \bigg{(} \theta^{**}\cdot \frac{T^{**}_t}{\theta^{**}} + (1 - \theta^{**})\cdot \frac{T^{**}_g}{1 - \theta^{**}}\bigg{)}\\
& =_{(c)} \frac{B_0}{G_0} \cdot \max\left(\frac{T^{**}_t}{\theta^{**}}, \frac{T^{**}_g}{1 -\theta^{**}}\right) \\
& =_{(d)} \frac{B_0}{G_0} \cdot \max\left(\frac{\eta^{**}_tm^{**}_t}{\theta^{**}}, \frac{\eta^{**}_gm_g^{**}}{1 -\theta^{**}}\right) \\
& = T_{\textsc{LlamaRL}}\left(b_t^{**}, b_g^{**}, m_t^{**}, m_g^{**}, \theta^{**}\right)
\end{split}
\end{eqnarray}
where inequality (a) follows from a direct algebra manipulation, inequality (b) follows from definition \eqref{decoupled_sols}, equality (c) is the  consequence of the third identity in \eqref{decoupled_property} of Lemma \ref{decoupled_optim}, equality (d) uses the first two identities in \eqref{decoupled_property} of Lemma \ref{decoupled_optim}, and the last equality is yielded by definition \eqref{llamarl_steptime}. Combining both sides of \eqref{last}, we conclude Theorem \ref{main_theorem}.
\end{proof}

\begin{remark}
    The final derivation in ~\eqref{last} helps identify two primary sources of efficiency gains. First, inequality (a) reflects how the freed-up GPU memory in the asynchronous setup allows for reduced model parallelism, which in turn lowers computation load on each GPU. Then, inequality (b) suggests that by decoupling trainer and generator memory constraints, the asynchronous framework can further optimize among \texttt{mp\_size} and micro batch size within each component independently to maximize overall step efficiency.
\end{remark}
\begin{remark}
    To simplify the proof of Theorem \ref{main_theorem}, we intentionally made several conservative assumptions. For example, we did not account for the potential use of quantization to reduce model size, which would further enable smaller \texttt{mp\_size}, nor did we argue for the additional communication savings associated with reduced \texttt{mp\_size}. Importantly, each of these omitted factors would only strengthen our conclusion by contributing to even greater efficiency gains.
\end{remark}

\section{Experiments}
\label{sec:exp}

We showcase the efficiency of \textsc{LlamaRL} with a number of  ablations.

\subsection{Experiment setup}
\label{exp_setup}
We empirically demonstrate the efficiency gains of \textsc{LlamaRL}, and show that it achieves comparable performance to the synchronous on-policy baseline on public benchmarks. Unless stated otherwise, all models are trained on the MATH dataset \citep{Hendrycks2021MeasuringMP}, a mathematical reasoning dataset that contains questions paired with short-form answers. Evaluations are conducted on the full test split of the MATH dataset \citep{Hendrycks2021MeasuringMP} as well as the GSM8K dataset \citep{Cobbe2021TrainingVT}. To further ensure fair evaluation and mitigate potential score inflation due to test set contamination in MATH, we follow  \citet{Lightman2023LetsVS} and additionally evaluate on a 500-problem held-out subset, commonly referred to as MATH-500.

\paragraph{\textbf{\textsc{LlamaRL} vs. Baseline.}}
\label{baseline_explanations}
To benchmark against \textsc{LlamaRL}, we adopt the commonly recognized synchronous on-policy baseline (see, e.g., DeepSpeed-Chat \citep{yao2023deepspeedchateasyfastaffordable}). Both \textsc{LlamaRL} and the baseline share identical optimizations for inference and training. Specifically, the inference optimization includes efficient memory management for KV-cache, optimized CUDA kernels, CUDA graph utilization \citep{nguyen2021cudagraphs}, and tensor parallelism \citep{shoeybi2019megatron}. The training side applies FSDP~\citep{zhao2023pytorchfsdpexperiencesscaling} to optimize memory usage. 

To reiterate, the key distinction between the asynchronous \textsc{LlamaRL} and the synchronous baseline, lies in their execution architecture: 
\begin{itemize}
    \item The baseline colocates inference and policy models within the same GPU cluster sharing the same model parallelism. In comparison, \textsc{LlamaRL} decouples these components to different GPU clusters with different parallelism and data precision. For example in Table~\ref{eff_results}, The 405B \textsc{LlamaRL} job can be incoorperated with generator with model parallelism size 8 to model parallelism size 32.
    \item The baseline execute the training and inference sequentially. In comparison, \textsc{LlamaRL} enables asynchronous training and inference with DDMA weights update across different components.
\end{itemize}

\paragraph{\textbf{Experimental configurations.}}
We conduct experiments on the 8B, 70B, and 405B variants of LLaMA 3.1 family of open sourced models \citep{Dubey2024TheL3}. For optimization, we use a fixed learning rate of $2 \times 10^{-7}$ with the Adam optimizer \citep{kingma2014adam,loshchilov2017decoupled}. For a fair comparison, both frameworks use identical RL algorithms, datasets, training hyperparameters, evaluation protocols, and the same number of H100 GPUs. We will specify all training details, including number of generations,  number of GPUs, model tensor parallelism, and batch sizes in experiments below. 

\subsection{Efficiency comparison}

\begin{table}[ht]
\centering
\small
\setlength\tabcolsep{3pt}
\begin{tabular}{ccccccccc}
\toprule
\multicolumn{1}{c|}{Model size (B)}  & \multicolumn{1}{c|}{\begin{tabular}[c]{@{}c@{}} Total step \\ time(s)\end{tabular}}            & \begin{tabular}[c]{@{}c@{}}Total\\ \# GPUs\end{tabular} & \begin{tabular}[c]{@{}c@{}} Generator \\ \# GPUs\end{tabular} & \multicolumn{1}{c}{\begin{tabular}[c]{@{}c@{}} Trainer \\ \# GPUs\end{tabular}} & \begin{tabular}[c]{@{}c@{}} Trainer \\ mp size\end{tabular} & \begin{tabular}[c]{@{}c@{}} Generator \\ mp size\end{tabular} &  
\begin{tabular}[c]{@{}c@{}}The global \\ batch size\end{tabular} & \multicolumn{1}{c}{\begin{tabular}[c]{@{}c@{}}Max decode\\  concurrency\end{tabular}} \\ \midrule
\multicolumn{9}{c}{\cellcolor[HTML]{EFEFEF}\bfseries{Baseline}}                                                                                                                                                                                                                                                                                                           \\ \midrule
\multicolumn{1}{c|}{8}    & \multicolumn{1}{c|}{22.45}  & {\color[HTML]{000000} 256}      & N/A   & N/A    & {\color[HTML]{000000} {8}}
& {\color[HTML]{000000} {8}}                                      & 2048                                                            & 16 x 16 (dp size)                                                                                 \\
\multicolumn{1}{c|}{70}    & \multicolumn{1}{c|}{82.32}   & {\color[HTML]{000000} 256}          & N/A   & N/A    & {\color[HTML]{000000} {8}} & {\color[HTML]{000000} {8}}                                                                        & 2048                                                            & 16 x 16                                                                                 \\
\multicolumn{1}{c|}{405}   & \multicolumn{1}{c|}{635.8}         & 1024  & N/A     & {\color[HTML]{000000} N/A} 
& {\color[HTML]{000000} 64} & {\color[HTML]{000000} 64}                                       & 2048                                                             & 32 x 16                                                                                   \\ \midrule
\multicolumn{9}{c}{\cellcolor[HTML]{EFEFEF}{\textsc{\bfseries{LlamaRL}}}}                                                                                                                                                                                                                                                                                                         \\ \midrule
\multicolumn{1}{c|}{8}    & \multicolumn{1}{c|}{\bfseries{12.22}}   & {\color[HTML]{000000} 256}         & 128   & {128}    & {\color[HTML]{000000} {8}}  & {\color[HTML]{000000} {8}}                                                                        & 2048                                                            & 64 x 16                                                                                 \\
\multicolumn{1}{c|}{8}    & \multicolumn{1}{c|}{\bfseries{8.90}}   & {\color[HTML]{000000} 256}    & 128   & {128}    & {\color[HTML]{000000} {8}}  & {\color[HTML]{000000} {1}}                                                                             & 2048                                                            & 32 x 128                                                                                 \\
\multicolumn{1}{c|}{70} & \multicolumn{1}{c|}{\bfseries{26.19}} & {\color[HTML]{000000} 256}  & 128   & {128}     & {\color[HTML]{000000} 8}  & {\color[HTML]{000000} 8}                                                                             & 2048                                                              & 64 x 16                                                                                                                                                                \\
\multicolumn{1}{c|}{70}    & \multicolumn{1}{c|}{\bfseries{20.67}}    & {\color[HTML]{000000} 256}    & 120   & 136    & {\color[HTML]{000000} {8}} & {\color[HTML]{000000} {4 (\texttt{fp8})}} & 2048 & 16 x 34  \\

\multicolumn{1}{c|}{405} & \multicolumn{1}{c|}{\bfseries{240.8}} & {\color[HTML]{000000} 1024} & {512}   & {512}     & {\color[HTML]{000000} 32} & {\color[HTML]{000000} 32}                                                                            & 2048                                                              & 32 x 16                            \\
\multicolumn{1}{c|}{405} & \multicolumn{1}{c|}{\bfseries{100.5}}  & {\color[HTML]{000000} 1024} & {512}   & {512}     & {16} & {\color[HTML]{000000} 16}                                                                            & 2048                                                              & 48 x 32                           \\
\multicolumn{1}{c|}{405} & \multicolumn{1}{c|}{\bfseries{59.5}}        & {\color[HTML]{000000} 1024} & {512}   & {512}     & {\color[HTML]{000000} 16} & {\color[HTML]{000000} 8 (\texttt{fp8})}                                                                      & 2048                                                              & 32 x 64                                                                                 \\ \bottomrule
\end{tabular}

\vspace{0.3em}
\caption{\small{RL speed comparison between \textsc{LlamaRL} and synchronized baseline. To ensure a fair comparison, all experiments use $n = 4$ completions per training prompt and $512$ unique prompts per training batch, yielding a consistent global batch size of 2048. We use 256 GPUs for both the 8B model and 70B model, and 1024 GPUs for the 405B model. For the baseline, both the generator and trainer are colocated on all GPUs, so their individual \# GPU is not listed.  For the max decode concurrency, we calculate it by multiplying the max decode concurrency per data parallel group with data parallel group size. }}
\label{eff_results}
\end{table}

\begin{table}[ht]
\centering
\small
\setlength\tabcolsep{3pt}
\begin{tabular}{c | c | c}
\toprule
Model Size (B) & 
OpenRLHF & 
\textsc{LlamaRL} \\
\midrule
7    & 4.32   & 0.04  \\
70   & 111.65 & 1.15   \\
405  &  -     & 2.31   \\
\bottomrule
\end{tabular}
\caption{The time consumption (seconds) of the weights synchronization in OpenRLHF and \textsc{LlamaRL}.}
\end{table}

Table~\ref{eff_results} summarizes the efficiency results. Under equal GPU budgets, \textsc{LlamaRL} achieves $2.52\times$, $3.98\times$ and $10.7\times$ speedups per RL step on the 8B, 70B and 405B models, respectively, demonstrating strong scalability with increasing model size.

While the largest gain at 405B partly benefits from \texttt{fp8} quantization on the generator, we emphasize that the primary source of improvement lies in \textsc{LlamaRL}'s parallelism flexibility. Specifically, it allows the generator and trainer to use different model parallel sizes. This enables us to set \texttt{mp\_size} = 8 for the generator while keeping trainer as bf16 at \texttt{mp\_size} = 16, effectively balancing generation and training time to maximize overall throughput. In contrast, the synchronous baseline (see Subsection~\ref{baseline_explanations}) requires a single global \texttt{mp\_size} for both generator and trainer. As a result, even if the baseline generator were quantized, it would still be constrained to \texttt{mp\_size} = 64, limiting the potential speedup.

We define the efficiency gain of \textsc{LlamaRL} as the ratio of RL step times between the synchronous baseline and \textsc{LlamaRL} under identical hardware and training settings. This gain—measured at 8B, 70B, and 405B model scales—is plotted against the logarithm of model size in Figure~\ref{figure:eff_scaling}. The convex trend indicates that the relative speedup grows super-linearly in log-scale as models increase in size. Consequently, \textsc{LlamaRL} demonstrates strong scalability and offers promising efficiency for training even larger language models in the future.

\begin{figure}[h]
\centering
\includegraphics[width=\textwidth, height=4.5cm]{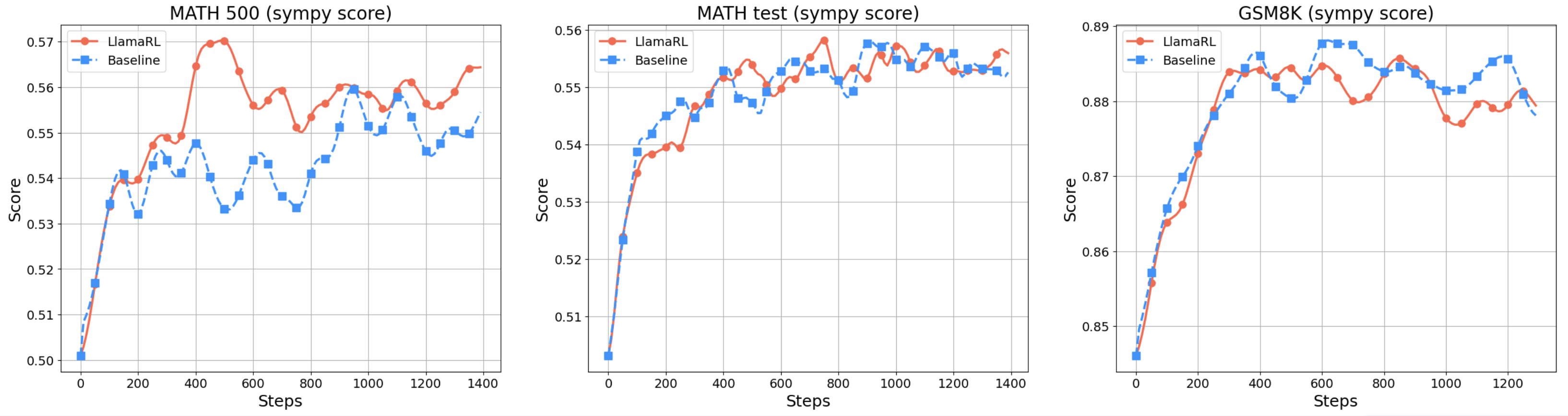}
\caption{\small{Evaluation results comparing \textsc{LlamaRL} with the synchronous RL baseline on MATH-500, MATH test, and GSM8K for 8B models. \textsc{LlamaRL} achieves comparable performance to the baseline across all benchmarks, while benefiting from a more efficient asynchronous training framework.}}
\label{figure:quality}
\end{figure}

\begin{figure}[ht]
\centering
\begin{minipage}[t]{0.33\textwidth}
\centering
\includegraphics[width=\textwidth, height=4.4cm]{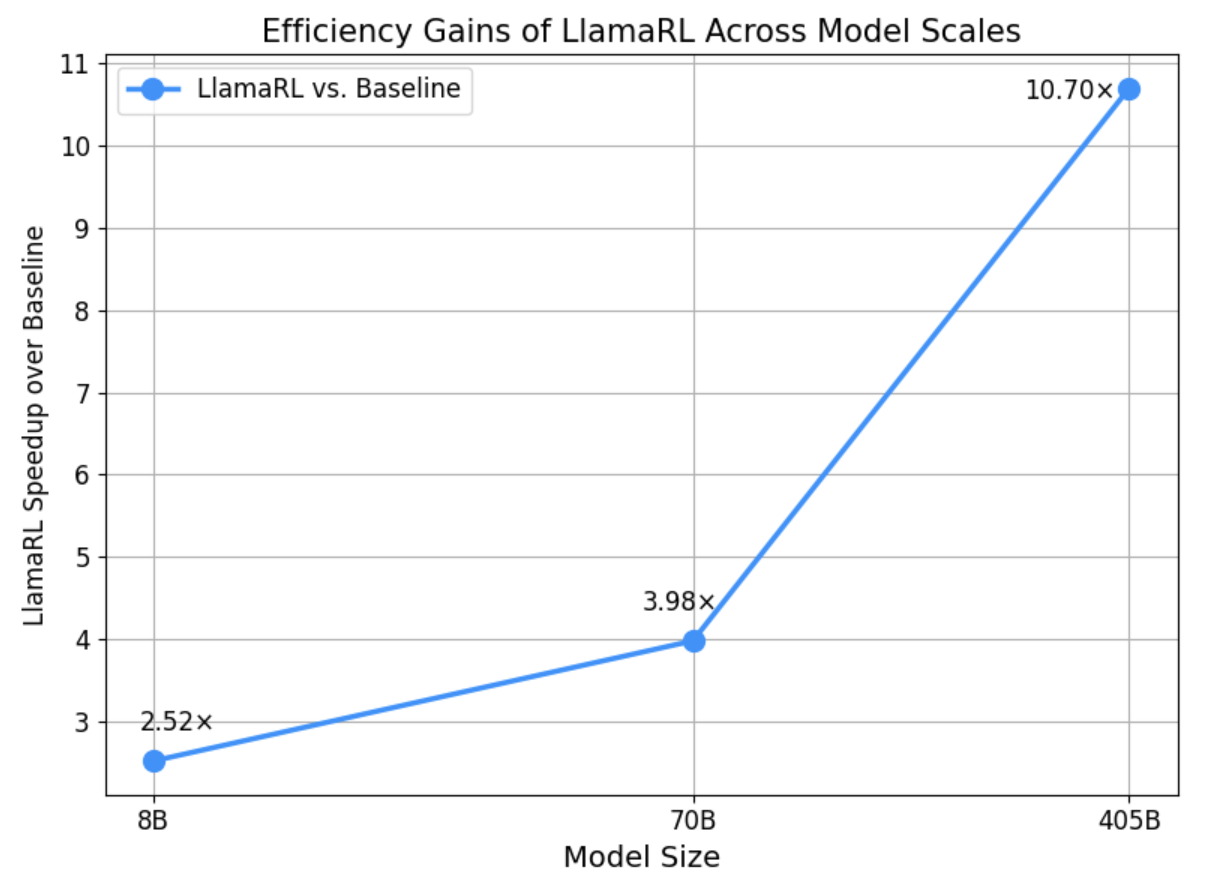}
\caption{\small{Efficiency gains of \textsc{LlamaRL} over the synchronous baseline at different model scales. Speedup increases with model size, reaching over 10× at 405B. We use a log-scaled x-axis to better visualize the growth trend.}}
\label{figure:eff_scaling}
\end{minipage}
\hfill
\begin{minipage}[t]{0.66\textwidth}
\centering
\includegraphics[width=\textwidth, height=4.4cm]{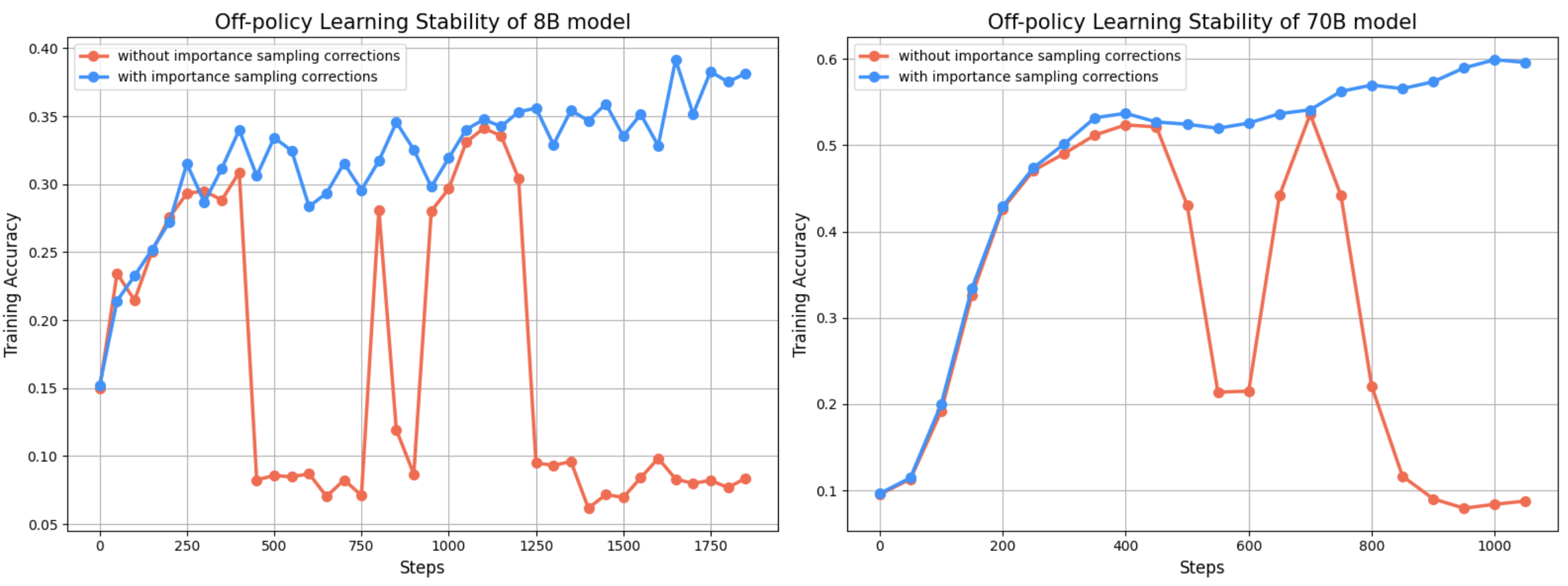}
\caption{\small{Ablation on the impact of off-policy corrections. We compare training stability with and without importance sampling corrections for asynchronous training in \textsc{LlamaRL}. With more sophisticated data mixtures, 8B and 70B models' training tend to be sporadically unstable. In contrast, applying off-policy corrections stabilizes training across model scales, demonstrating their necessity for ensuring robustness in asynchronous RL training framework.}}
\label{figure:off-policy-corrections}
\end{minipage}
\end{figure}

\subsection{Quality comparison to on-policy learning}
To assess the training quality under different RL training frameworks, we evaluate on the MATH test set, MATH-500 (see Subsection~\ref{exp_setup}), and GSM8K, using the \textit{sympy score} as our primary metric. This score measures answer accuracy by parsing both predicted and reference expressions with the \texttt{sympy} symbolic algebra library and checking for symbolic equivalence \citep{10.7717/peerj-cs.103}. We use the same score as reward for training applied to the training data.

We train 8B models using the \textsc{LlamaRL} framework with the \texttt{AIPO} algorithm. As a strong and fair comparison, we use the synchronous RL baseline described earlier in Subsection~\ref{baseline_explanations}. Algorithmically, the baseline can be understood as an on-policy version of \texttt{AIPO} - the same gradient estimator but without importance sampling corrections, since the baseline is already strictly on-policy. All hyper-parameters, including the global batch size and learning rate, are kept identical across both settings.

Figure~\ref{figure:quality} presents the evaluation results on MATH-500, the full MATH test set, and GSM8K. On both the MATH test set and GSM8K, \textsc{LlamaRL} performs on par with the traditional synchronous RL framework, demonstrating that the asynchronous training does not compromise training quality. Interestingly, \textsc{LlamaRL} shows a consistent improvement on MATH-500 throughout training. While this gain may be partially attributed to statistical variation, it may also reflect potential benefits of off-policy learning over strictly on-policy approaches (see, e.g.,  \citet{Lambert2024TLU3P} on this point). We leave a deeper investigation of this observation to future work.

\paragraph{\textbf{Remark.}} We focus on the 8B model for evaluation, as the 70B model does not obtain significant improvements by only training on the MATH training set. The 8B model evaluation allows us to observe clearer learning trends and more effectively compare training quality between RL frameworks.

\section{Conclusion}
In this paper, we introduced \textsc{LlamaRL}, a fully distributed and asynchronous reinforcement learning framework for large-scale training of large language models (LLMs). By combining asynchronous training with diverse parallelization strategies and mixed-precision techniques across various RL components, \textsc{LlamaRL} dramatically improves computational and memory efficiency. Our proposed Distributed Direct Memory Access (DDMA) method further ensures efficient weight updates among actors. Consequently, \textsc{LlamaRL} exhibits robust scalability—from a small cluster of H100 GPUs to thousands or even more H100 GPUs—providing a practical path forward as model sizes and training demands continue to grow. 

We establish the efficiency gains of \textsc{LlamaRL} both theoretically and empirically. From a theoretical perspective, we provide a formal analysis and prove that \textsc{LlamaRL} achieves strictly lower step time than any synchronous RL framework under the same resource constraints. Our experiments indicate that off-policy corrections effectively address potential drawbacks of asynchronous learning, preserving both training stability and final performance. In particular, \textsc{LlamaRL} achieves a 10.7× increase in RL training throughput on extremely large (405B) models. 

Overall, \textsc{LlamaRL} marks the next step in RL-based fine-tuning for LLMs, simplifying the integration of complex  algorithms involing potentially multiple models at scale. We anticipate that our approach will benefit both researchers and practitioners by delivering greater efficiency, enhanced reliability, and easier extensibility in RL pipelines for modern AI workloads. Future directions include exploring more advanced off-policy learning methods, integrating multi-task objectives, and extending asynchronous RL to broader modalities.

\newpage
\section*{Contributors and Acknowledgement}

\subsection*{Co-first authors}
Bo Wu, Sid Wang, Yunhao Tang, Jia Ding

\subsection*{Core contributors}
Eryk Helenowski, Liang Tan, Tengyu Xu, Tushar Gowda, Zhengxing Chen (in alphabetical order of the first name)

\subsection*{Other contributors}
Chen Zhu, Xiaocheng Tang, Yundi Qian (in alphabetical order of the first name)

\subsection*{Project Leads}
Beibei Zhu, Rui Hou (in alphabetical order of the first name)

\subsection*{Acknowledgement}
Thanks to the following members from the Llama team who continually work on top of our initial \textsc{LlamaRL} stack, contributing useful features and providing valuable feedback: Qian Sun, Jessica Zhong, Boduo Li, Jemma Fan, Ning Li, Zhitong Guo, Hongbo Zhang, Haiping Zhao, Naman Goyal, Dhruv Mahajan, Rohan Anil, Kushal Lakhotia, Ruan Silva, Sagar Jain, Licheng Yu, Yuanhao Xiong, Karthik Prasad, Di Jin, Congying Xia, Rohan Maheshwari, Jinxi Guo, Xiaopeng Li, Damien Allonsius, David Brandfonbrener, Xuewei Wang, Lovish Madaan, Anirudh Goyal, Andy Su, Arvind Neelakantan, Yuandong Tian. Thanks to the members from Meta FAIR for discussions: R\'emi Munos, Taco Cohen, Kunhao Zheng. Thanks to the following members from the Meta infrastructure team to provide strong underlying infrastructure and seamless support: Weiwei Chu, Sheng Shen, Naman Goyal, Xiaodong Wang, Jianyu Huang, Jiecao Yu, Jenya Lee, Rohan Verma, Xinfeng Xie.

\clearpage
\newpage
\bibliographystyle{apalike}
\bibliography{paper}

\begin{thebibliography}{}

\bibitem[Ahmadian et~al., 2024]{ahmadian2024back}
Ahmadian, A., Cremer, C., Gall{\'e}, M., Fadaee, M., Kreutzer, J., Pietquin,
  O., {\"U}st{\"u}n, A., and Hooker, S. (2024).
\newblock Back to basics: Revisiting reinforce style optimization for learning
  from human feedback in llms.
\newblock {\em arXiv preprint arXiv:2402.14740}.

\bibitem[Anthropic, 2023]{Anthropic2023Claude}
Anthropic (2023).
\newblock Claude: Training helpful and harmless ai assistants.
\newblock {\em Anthropic AI Blog}.
\newblock URL: \url{https://www.anthropic.com/index/2023/Claude}.

\bibitem[Bai et~al., 2022a]{Bai2022Training}
Bai, Y., Jones, A., Ndousse, K., and et~al. (2022a).
\newblock Training a helpful and harmless assistant with rlhf.
\newblock {\em arXiv preprint arXiv:2204.05862}.

\bibitem[Bai et~al., 2022b]{Bai2022ConstitutionalAI}
Bai, Y., Jones, A., Ndousse, K., and et~al. (2022b).
\newblock Training a helpful and harmless assistant with rlhf.
\newblock {\em arXiv preprint arXiv:2204.05862}.

\bibitem[Berner et~al., 2019]{berner2019dota}
Berner, C., Brockman, G., Chan, B., Cheung, V., D{\k{e}}biak, P., Dennison, C.,
  Farhi, D., Fischer, Q., Hashme, S., Hesse, C., et~al. (2019).
\newblock Dota 2 with large scale deep reinforcement learning.
\newblock {\em arXiv preprint arXiv:1912.06680}.

\bibitem[Brown et~al., 2020]{brown2020language}
Brown, T., Mann, B., Ryder, N., Subbiah, M., Kaplan, J.~D., Dhariwal, P.,
  Neelakantan, A., Shyam, P., Sastry, G., Askell, A., et~al. (2020).
\newblock Language models are few-shot learners.
\newblock {\em Advances in neural information processing systems},
  33:1877--1901.

\bibitem[Chilimbi et~al., 2014]{chilimbi2014adam}
Chilimbi, T., Suzue, Y., Apacible, J., and Kalyanaraman, K. (2014).
\newblock {Project Adam}: Building an efficient and scalable deep learning
  training system.
\newblock In {\em Proceedings of the 11th USENIX Symposium on Operating Systems
  Design and Implementation (OSDI)}, pages 571--582.

\bibitem[Christiano et~al., 2017]{Christiano2017}
Christiano, P., Leike, J., Brown, T., and et~al. (2017).
\newblock Deep reinforcement learning from human preferences.
\newblock In {\em Advances in Neural Information Processing Systems (NeurIPS)}.

\bibitem[Cobbe et~al., 2021]{Cobbe2021TrainingVT}
Cobbe, K., Kosaraju, V., Bavarian, M., Chen, M., Jun, H., Kaiser, L., Plappert,
  M., Tworek, J., Hilton, J., Nakano, R., Hesse, C., and Schulman, J. (2021).
\newblock Training verifiers to solve math word problems.
\newblock {\em ArXiv}, abs/2110.14168.

\bibitem[DeepSeek-AI, 2025]{deepseek2025r1}
DeepSeek-AI (2025).
\newblock Deepseek-r1: Incentivizing reasoning capability in llms via
  reinforcement learning.
\newblock {\em arXiv preprint arXiv:2501.12948}.

\bibitem[Dong et~al., 2023]{dong2023raft}
Dong, H., Xiong, W., Goyal, D., Zhang, Y., Chow, W., Pan, R., Diao, S., Zhang,
  J., Shum, K., and Zhang, T. (2023).
\newblock Raft: Reward ranked finetuning for generative foundation model
  alignment.
\newblock {\em arXiv preprint arXiv:2304.06767}.

\bibitem[Dubey et~al., 2024]{Dubey2024TheL3}
Dubey, A., Jauhri, A., Pandey, A., Kadian, A., Al-Dahle, A., Letman, A.,
  Mathur, A., Schelten, A., Yang, A., Fan, A., Goyal, A., Hartshorn, A.~S.,
  Yang, A., Mitra, A., Sravankumar, A., Korenev, A., Hinsvark, A., Rao, A.,
  Zhang, A., Rodriguez, A., Gregerson, A., Spataru, A., tiste Rozi{\`e}re, B.,
  Biron, B., Tang, B., Chern, B., Caucheteux, C., Nayak, C., Bi, C., Marra, C.,
  McConnell, C., Keller, C., Touret, C., Wu, C., Wong, C., Ferrer, C.~C.,
  Nikolaidis, C., Allonsius, D., Song, D., Pintz, D., Livshits, D., Esiobu, D.,
  Choudhary, D., Mahajan, D., Garcia-Olano, D., Perino, D., Hupkes, D.,
  Lakomkin, E., AlBadawy, E.~A., Lobanova, E., Dinan, E., Smith, E.~M.,
  Radenovic, F., Zhang, F., Synnaeve, G., Lee, G., Anderson, G.~L., Nail, G.,
  Mialon, G., Pang, G., Cucurell, G., Nguyen, H., Korevaar, H., Xu, H.,
  Touvron, H., Zarov, I., Ibarra, I.~A., Kloumann, I.~M., Misra, I., Evtimov,
  I., Copet, J., Lee, J., Geffert, J., Vranes, J., Park, J., Mahadeokar, J.,
  Shah, J., van~der Linde, J., Billock, J., Hong, J., Lee, J., Fu, J., Chi, J.,
  Huang, J., Liu, J., Wang, J., Yu, J., Bitton, J., Spisak, J., Park, J.,
  Rocca, J., Johnstun, J., Saxe, J., Jia, J.-Q., Alwala, K.~V., Upasani, K.,
  Plawiak, K., Li, K., neth Heafield, K.-., Stone, K., El-Arini, K., Iyer, K.,
  Malik, K., ley Chiu, K., Bhalla, K., Rantala-Yeary, L., van~der Maaten, L.,
  Chen, L., Tan, L., Jenkins, L., Martin, L., Madaan, L., Malo, L., Blecher,
  L., Landzaat, L., de~Oliveira, L., Muzzi, M., Pasupuleti, M., Singh, M.,
  Paluri, M., Kardas, M., Oldham, M., Rita, M., Pavlova, M., Kambadur, M.
  H.~M., Lewis, M., Si, M., Singh, M.~K., Hassan, M., Goyal, N., Torabi, N.,
  Bashlykov, N., Bogoychev, N., Chatterji, N.~S., Duchenne, O., cCelebi, O.,
  Alrassy, P., Zhang, P., Li, P., Vasi{\'c}, P., Weng, P., Bhargava, P., Dubal,
  P., Krishnan, P., Koura, P.~S., Xu, P., He, Q., Dong, Q., Srinivasan, R.,
  Ganapathy, R., Calderer, R., Cabral, R.~S., Stojnic, R., Raileanu, R.,
  Girdhar, R., Patel, R., main Sauvestre, R., Polidoro, R., Sumbaly, R.,
  Taylor, R., Silva, R., Hou, R., Wang, R., Hosseini, S., Chennabasappa, S.,
  Singh, S., Bell, S., Kim, S.~S., Edunov, S., Nie, S., Narang, S., Raparthy,
  S.~C., Shen, S., Wan, S., Bhosale, S., Zhang, S., Vandenhende, S., Batra, S.,
  Whitman, S., Sootla, S., Collot, S., Gururangan, S., Borodinsky, S., Herman,
  T., Fowler, T., Sheasha, T., Georgiou, T., Scialom, T., Speckbacher, T.,
  Mihaylov, T., Xiao, T., Karn, U., Goswami, V., Gupta, V., Ramanathan, V.,
  Kerkez, V., Gonguet, V., Do, V., Vogeti, V., Petrovic, V., Chu, W., Xiong,
  W., Fu, W., ney Meers, W., Martinet, X., Wang, X., Tan, X.~E., Xie, X., Jia,
  X., Wang, X., Goldschlag, Y., Gaur, Y., Babaei, Y., Wen, Y., Song, Y., Zhang,
  Y., Li, Y., Mao, Y., Coudert, Z.~D., Yan, Z., Chen, Z., Papakipos, Z., Singh,
  A.~K., Grattafiori, A., Jain, A., Kelsey, A., Shajnfeld, A., Gangidi, A.,
  Victoria, A., Goldstand, A., Menon, A., Sharma, A., Boesenberg, A., Vaughan,
  A., Baevski, A., Feinstein, A., Kallet, A., Sangani, A., Yunus, A., Lupu, A.,
  Alvarado, A., Caples, A., Gu, A., Ho, A., Poulton, A., Ryan, A., Ramchandani,
  A., Franco, A., Saraf, A., Chowdhury, A., Gabriel, A., Bharambe, A.,
  Eisenman, A., Yazdan, A., James, B., Maurer, B., Leonhardi, B., Huang,
  P.-Y.~B., Loyd, B., Paola, B.~D., Paranjape, B., Liu, B., Wu, B., Ni, B.,
  Hancock, B., Wasti, B., Spence, B., Stojkovic, B., Gamido, B., Montalvo, B.,
  Parker, C., Burton, C., Mejia, C., Wang, C., Kim, C., Zhou, C., Hu, C., Chu,
  C.-H., Cai, C., Tindal, C., Feichtenhofer, C., Civin, D., Beaty, D., Kreymer,
  D., Li, S.-W., Wyatt, D., Adkins, D., Xu, D., Testuggine, D., David, D.,
  Parikh, D., Liskovich, D., Foss, D., Wang, D., Le, D., Holland, D., Dowling,
  E., Jamil, E., Montgomery, E., Presani, E., Hahn, E., Wood, E., Brinkman, E.,
  Arcaute, E., Dunbar, E., Smothers, E., Sun, F., Kreuk, F., Tian, F., Ozgenel,
  F., Caggioni, F., Guzm'an, F., Kanayet, F.~J., Seide, F., Florez, G.~M.,
  Schwarz, G., Badeer, G., Swee, G., Halpern, G., Thattai, G., Herman, G.,
  Sizov, G.~G., Zhang, G., Lakshminarayanan, G., Shojanazeri, H., Zou, H.,
  Wang, H., Zha, H., Habeeb, H., Rudolph, H., Suk, H., Aspegren, H., Goldman,
  H., Molybog, I., Tufanov, I., Veliche, I.-E., Gat, I., Weissman, J., Geboski,
  J., Kohli, J., Asher, J., Gaya, J.-B., Marcus, J., Tang, J., Chan, J., Zhen,
  J., Reizenstein, J., Teboul, J., Zhong, J., Jin, J., Yang, J., Cummings, J.,
  Carvill, J., Shepard, J., McPhie, J., Torres, J., Ginsburg, J., Wang, J., Wu,
  K., KamHou, U., Saxena, K., Prasad, K., Khandelwal, K., Zand, K., Matosich,
  K., Veeraraghavan, K., Michelena, K., Li, K., Huang, K., Chawla, K.,
  Lakhotia, K., Huang, K., Chen, L., Garg, L., Lavender, A., Silva, L., Bell,
  L., Zhang, L., Guo, L., Yu, L., Moshkovich, L., Wehrstedt, L., Khabsa, M.,
  Avalani, M., Bhatt, M., Tsimpoukelli, M., Mankus, M., Hasson, M., Lennie, M.,
  Reso, M., Groshev, M., Naumov, M., Lathi, M., Keneally, M., Seltzer, M.~L.,
  Valko, M., Restrepo, M., Patel, M., Vyatskov, M., Samvelyan, M., Clark, M.,
  Macey, M., Wang, M., Hermoso, M.~J., Metanat, M., Rastegari, M., Bansal, M.,
  Santhanam, N., Parks, N., White, N., Bawa, N., Singhal, N., Egebo, N.,
  Usunier, N., Laptev, N.~P., Dong, N., Zhang, N., Cheng, N., Chernoguz, O.,
  Hart, O., Salpekar, O., Kalinli, O., Kent, P., Parekh, P., Saab, P., Balaji,
  P., Rittner, P., Bontrager, P., Roux, P., Doll{\'a}r, P., Zvyagina, P.,
  Ratanchandani, P., Yuvraj, P., Liang, Q., Alao, R., Rodriguez, R., Ayub, R.,
  Murthy, R., Nayani, R., Mitra, R., Li, R., Hogan, R., Battey, R., Wang, R.,
  Maheswari, R., Howes, R., Rinott, R., Bondu, S.~J., Datta, S., Chugh, S.,
  Hunt, S., Dhillon, S., Sidorov, S., Pan, S., Verma, S., Yamamoto, S.,
  Ramaswamy, S., Lindsay, S., Feng, S., Lin, S., Zha, S.~C., Shankar, S.,
  Zhang, S., Wang, S., Agarwal, S., Sajuyigbe, S., Chintala, S., Max, S., Chen,
  S., Kehoe, S., Satterfield, S., Govindaprasad, S., Gupta, S., Cho, S.-B.,
  Virk, S., Subramanian, S., Choudhury, S., Goldman, S., Remez, T., Glaser, T.,
  Best, T., Kohler, T., Robinson, T., Li, T., Zhang, T., Matthews, T., Chou,
  T., Shaked, T., Vontimitta, V., Ajayi, V., Montanez, V., Mohan, V., Kumar,
  V.~S., Mangla, V., Ionescu, V., Poenaru, V.~A., Mihailescu, V.~T., Ivanov,
  V., Li, W., Wang, W., Jiang, W., Bouaziz, W., Constable, W., Tang, X., Wang,
  X., Wu, X., Wang, X., Xia, X., Wu, X., Gao, X., Chen, Y., Hu, Y., Jia, Y.,
  Qi, Y., Li, Y., Zhang, Y., Zhang, Y., Adi, Y., Nam, Y., Wang, Y., Hao, Y.,
  Qian, Y., He, Y., Rait, Z., DeVito, Z., Rosnbrick, Z., Wen, Z., Yang, Z., and
  Zhao, Z. (2024).
\newblock The llama 3 herd of models.
\newblock {\em ArXiv}, abs/2407.21783.

\bibitem[El-Kishky et~al., 2025]{el2025competitive}
El-Kishky, A., Wei, A., Saraiva, A., Minaev, B., Selsam, D., Dohan, D., Song,
  F., Lightman, H., Clavera, I., Pachocki, J., et~al. (2025).
\newblock Competitive programming with large reasoning models.
\newblock {\em arXiv preprint arXiv:2502.06807}.

\bibitem[Espeholt et~al., 2018]{espeholt2018impala}
Espeholt, L., Soyer, H., Munos, R., Simonyan, K., Mnih, V., Ward, T., Doron,
  Y., Firoiu, V., Harley, T., Dunning, I., et~al. (2018).
\newblock Impala: Scalable distributed deep-rl with importance weighted
  actor-learner architectures.
\newblock In {\em International conference on machine learning}, pages
  1407--1416. PMLR.

\bibitem[Gehring et~al., 2024]{gehring2024rlef}
Gehring, J., Zheng, K., Copet, J., Mella, V., Cohen, T., and Synnaeve, G.
  (2024).
\newblock Rlef: Grounding code llms in execution feedback with reinforcement
  learning.
\newblock {\em arXiv preprint arXiv:2410.02089}.

\bibitem[Google, 2023]{Google2023Bard}
Google (2023).
\newblock Bard: Conversational ai by google.
\newblock {\em Google AI Blog}.
\newblock URL: \url{https://blog.google/technology/ai/bard-google-ai/}.

\bibitem[Guan et~al., 2024]{guan2024deliberative}
Guan, M.~Y., Joglekar, M., Wallace, E., Jain, S., Barak, B., Helyar, A., Dias,
  R., Vallone, A., Ren, H., Wei, J., et~al. (2024).
\newblock Deliberative alignment: Reasoning enables safer language models.
\newblock {\em arXiv preprint arXiv:2412.16339}.

\bibitem[Hendrycks et~al., 2021]{Hendrycks2021MeasuringMP}
Hendrycks, D., Burns, C., Kadavath, S., Arora, A., Basart, S., Tang, E., Song,
  D.~X., and Steinhardt, J. (2021).
\newblock Measuring mathematical problem solving with the math dataset.
\newblock {\em ArXiv}, abs/2103.03874.

\bibitem[Hestness et~al., 2017]{Hestness2017Scaling}
Hestness, J., Narang, S., Ardalani, N., Diamos, G., Jun, H., Kianinejad, H.,
  Patwary, M. M.~A., Yang, Y., and Zhou, Y. (2017).
\newblock Deep learning scaling is predictable, empirically.
\newblock {\em arXiv preprint arXiv:1712.00409}.

\bibitem[Hoffmann et~al., 2022]{hoffmann2022training}
Hoffmann, J., Borgeaud, S., Mensch, A., Buchatskaya, E., Cai, T., Rutherford,
  E., Casas, D. d.~L., Hendricks, L.~A., Welbl, J., Clark, A., et~al. (2022).
\newblock Training compute-optimal large language models.
\newblock {\em arXiv preprint arXiv:2203.15556}.

\bibitem[Hu et~al., 2024]{hu2024openrlhfeasytousescalablehighperformance}
Hu, J., Wu, X., Wang, W., Xianyu, Zhang, D., and Cao, Y. (2024).
\newblock Openrlhf: An easy-to-use, scalable and high-performance rlhf
  framework.

\bibitem[Jaech et~al., 2024]{jaech2024openai}
Jaech, A., Kalai, A., Lerer, A., Richardson, A., El-Kishky, A., Low, A.,
  Helyar, A., Madry, A., Beutel, A., Carney, A., et~al. (2024).
\newblock Openai o1 system card.
\newblock {\em arXiv preprint arXiv:2412.16720}.

\bibitem[Jie and Abbeel, 2010]{jie2010connection}
Jie, T. and Abbeel, P. (2010).
\newblock On a connection between importance sampling and the likelihood ratio
  policy gradient.
\newblock {\em Advances in Neural Information Processing Systems}, 23.

\bibitem[Kakade and Langford, 2002]{kakade2002approximately}
Kakade, S. and Langford, J. (2002).
\newblock Approximately optimal approximate reinforcement learning.
\newblock In {\em Proceedings of the nineteenth international conference on
  machine learning}, pages 267--274.

\bibitem[Kaplan et~al., 2020]{Kaplan2020Scaling}
Kaplan, J., McCandlish, S., Henighan, T., and et~al. (2020).
\newblock Scaling laws for neural language models.
\newblock {\em arXiv preprint arXiv:2001.08361}.

\bibitem[Kapturowski et~al., 2018]{kapturowski2018recurrent}
Kapturowski, S., Ostrovski, G., Quan, J., Munos, R., and Dabney, W. (2018).
\newblock Recurrent experience replay in distributed reinforcement learning.
\newblock In {\em International conference on learning representations}.

\bibitem[Kazemnejad et~al., 2024]{kazemnejad2024vineppo}
Kazemnejad, A., Aghajohari, M., Portelance, E., Sordoni, A., Reddy, S.,
  Courville, A., and Roux, N.~L. (2024).
\newblock Vineppo: Unlocking rl potential for llm reasoning through refined
  credit assignment.
\newblock {\em arXiv preprint arXiv:2410.01679}.

\bibitem[{Kimi Team}, 2025]{kimi2025}
{Kimi Team} (2025).
\newblock Kimi k1.5: Scaling reinforcement learning with llms.
\newblock {\em arXiv preprint arXiv:2501.12599}.
\newblock Version 2, revised on March 5, 2025.

\bibitem[Kingma and Ba, 2014]{kingma2014adam}
Kingma, D.~P. and Ba, J. (2014).
\newblock Adam: A method for stochastic optimization.
\newblock {\em arXiv preprint arXiv:1412.6980}.

\bibitem[Lambert et~al., 2024]{Lambert2024TLU3P}
Lambert, N., Morrison, J.~D., Pyatkin, V., Huang, S., Ivison, H., Brahman, F.,
  Miranda, L. J.~V., Liu, A., Dziri, N., Lyu, X., Gu, Y., Malik, S., Graf, V.,
  Hwang, J.~D., Yang, J., Bras, R.~L., Tafjord, O., Wilhelm, C., Soldaini, L.,
  Smith, N.~A., Wang, Y., Dasigi, P., and Hajishirzi, H. (2024).
\newblock T{\"u}lu 3: Pushing frontiers in open language model post-training.
\newblock {\em ArXiv}, abs/2411.15124.

\bibitem[Li et~al., 2014]{li2014scaling_ps}
Li, M., Andersen, D.~G., Park, J.~W., Smola, A.~J., Ahmed, A., Josifovski, V.,
  and Long, J. (2014).
\newblock Scaling distributed machine learning with the parameter server.
\newblock In {\em Proceedings of the 11th USENIX Symposium on Operating Systems
  Design and Implementation (OSDI)}, pages 583--598.

\bibitem[Li et~al., 2022]{li2022competition}
Li, Y., Choi, D., Chung, J., Kushman, N., Schrittwieser, J., Leblond, R.,
  Eccles, T., Keeling, J., Gimeno, F., Dal~Lago, A., et~al. (2022).
\newblock Competition-level code generation with alphacode.
\newblock {\em Science}, 378(6624):1092--1097.

\bibitem[Li et~al., 2024]{li2024remaxsimpleeffectiveefficient}
Li, Z., Xu, T., Zhang, Y., Lin, Z., Yu, Y., Sun, R., and Luo, Z.-Q. (2024).
\newblock Remax: A simple, effective, and efficient reinforcement learning
  method for aligning large language models.

\bibitem[Lightman et~al., 2023a]{lightman2023let}
Lightman, H., Kosaraju, V., Burda, Y., Edwards, H., Baker, B., Lee, T., Leike,
  J., Schulman, J., Sutskever, I., and Cobbe, K. (2023a).
\newblock Let's verify step by step.
\newblock {\em arXiv preprint arXiv:2305.20050}.

\bibitem[Lightman et~al., 2023b]{Lightman2023LetsVS}
Lightman, H., Kosaraju, V., Burda, Y., Edwards, H., Baker, B., Lee, T., Leike,
  J., Schulman, J., Sutskever, I., and Cobbe, K. (2023b).
\newblock Let's verify step by step.
\newblock {\em ArXiv}, abs/2305.20050.

\bibitem[{Llama Team}, 2024]{grattafiori2024llama}
{Llama Team} (2024).
\newblock The llama 3 herd of models.
\newblock {\em arXiv preprint arXiv:2407.21783}.

\bibitem[Loshchilov and Hutter, 2017]{loshchilov2017decoupled}
Loshchilov, I. and Hutter, F. (2017).
\newblock Decoupled weight decay regularization.
\newblock {\em arXiv preprint arXiv:1711.05101}.

\bibitem[Meurer et~al., 2017]{10.7717/peerj-cs.103}
Meurer, A., Smith, C.~P., Paprocki, M., \v{C}ert\'{i}k, O., Kirpichev, S.~B.,
  Rocklin, M., Kumar, A., Ivanov, S., Moore, J.~K., Singh, S., Rathnayake, T.,
  Vig, S., Granger, B.~E., Muller, R.~P., Bonazzi, F., Gupta, H., Vats, S.,
  Johansson, F., Pedregosa, F., Curry, M.~J., Terrel, A.~R., Rou\v{c}ka, v.,
  Saboo, A., Fernando, I., Kulal, S., Cimrman, R., and Scopatz, A. (2017).
\newblock Sympy: symbolic computing in python.
\newblock {\em PeerJ Computer Science}, 3:e103.

\bibitem[Mnih et~al., 2016]{mnih2016asynchronous}
Mnih, V., Badia, A.~P., Mirza, M., Graves, A., Lillicrap, T., Harley, T.,
  Silver, D., and Kavukcuoglu, K. (2016).
\newblock Asynchronous methods for deep reinforcement learning.
\newblock In {\em International conference on machine learning}, pages
  1928--1937. PmLR.

\bibitem[Moritz et~al., 2018]{moritz2018ray}
Moritz, P., Nishihara, R., Wang, S., Tumanov, A., Liaw, R., Liang, E., Elibol,
  M., Yang, Z., Paul, W., Jordan, M.~I., and Stoica, I. (2018).
\newblock {Ray: A Distributed Framework for Emerging AI Applications}.
\newblock In {\em Proceedings of the 13th USENIX Symposium on Operating Systems
  Design and Implementation (OSDI)}, pages 561--577, Carlsbad, CA. USENIX
  Association.

\bibitem[Munos et~al., 2016]{munos2016safe}
Munos, R., Stepleton, T., Harutyunyan, A., and Bellemare, M. (2016).
\newblock Safe and efficient off-policy reinforcement learning.
\newblock {\em Advances in neural information processing systems}, 29.

\bibitem[Nguyen et~al., 2021]{nguyen2021cudagraphs}
Nguyen, V., Carilli, M., Eryilmaz, S.~B., Singh, V., Lin, M., Gimelshein, N.,
  Desmaison, A., and Yang, E. (2021).
\newblock Accelerating {PyTorch} with {CUDA} graphs.
\newblock
  \url{https://pytorch.org/blog/accelerating-pytorch-with-cuda-graphs/}.

\bibitem[Noukhovitch et~al.,
  2024]{noukhovitch2024asynchronousrlhffasterefficient}
Noukhovitch, M., Huang, S., Xhonneux, S., Hosseini, A., Agarwal, R., and
  Courville, A. (2024).
\newblock Asynchronous rlhf: Faster and more efficient off-policy rl for
  language models.

\bibitem[OpenAI, 2023]{OpenAI2023GPT4}
OpenAI (2023).
\newblock Gpt-4 technical report.
\newblock {\em arXiv preprint arXiv:2303.08774}.

\bibitem[Ouyang et~al., 2022]{Ouyang2022RLHF}
Ouyang, L., Wu, J., Jiang, X., and et~al. (2022).
\newblock Training language models to follow instructions with human feedback.
\newblock {\em arXiv preprint arXiv:2203.02155}.

\bibitem[Potluri and et~al., 2013]{potluri2013gpudirect}
Potluri, S. and et~al. (2013).
\newblock {GPU Direct RDMA} for infiniband on {NVIDIA} {GPUs}: A case study.
\newblock In {\em Proceedings of the 2013 IEEE International Symposium on
  Cluster, Cloud and Grid Computing (CCGrid)}, pages 340--347.

\bibitem[Precup et~al., 2000]{precup2000eligibility}
Precup, D., Sutton, R.~S., and Singh, S. (2000).
\newblock Eligibility traces for off-policy policy evaluation.
\newblock In {\em ICML}, volume 2000, pages 759--766. Citeseer.

\bibitem[Rajbhandari et~al., 2019a]{rajbhandari2019zero}
Rajbhandari, S., Rasley, J., Ruwase, O., and He, Y. (2019a).
\newblock Zero: Memory optimizations toward training trillion parameter models.

\bibitem[Rajbhandari et~al., 2019b]{Rajbhandari2019ZeROMO}
Rajbhandari, S., Rasley, J., Ruwase, O., and He, Y. (2019b).
\newblock Zero: Memory optimizations toward training trillion parameter models.
\newblock {\em SC20: International Conference for High Performance Computing,
  Networking, Storage and Analysis}, pages 1--16.

\bibitem[Rajbhandari et~al., 2020]{rajbhandari2020zero}
Rajbhandari, S., Rasley, J., Ruwase, O., and He, Y. (2020).
\newblock Zero: Memory optimization towards training trillion parameter models.
\newblock {\em Proceedings of Machine Learning and Systems}.

\bibitem[Schulman et~al., 2015]{schulman2015trust}
Schulman, J., Levine, S., Abbeel, P., Jordan, M., and Moritz, P. (2015).
\newblock Trust region policy optimization.
\newblock In {\em International conference on machine learning}, pages
  1889--1897. PMLR.

\bibitem[Schulman et~al., 2017]{Schulman2017PPO}
Schulman, J., Wolski, F., Dhariwal, P., Radford, A., and Klimov, O. (2017).
\newblock Proximal policy optimization algorithms.
\newblock {\em arXiv preprint arXiv:1707.06347}.

\bibitem[Shao et~al., 2024]{shao2024deepseekmath}
Shao, Z., Wang, P., Zhu, Q., Xu, R., Song, J., Bi, X., Zhang, H., Zhang, M.,
  Li, Y., Wu, Y., et~al. (2024).
\newblock Deepseekmath: Pushing the limits of mathematical reasoning in open
  language models.
\newblock {\em arXiv preprint arXiv:2402.03300}.

\bibitem[Shen et~al., 2024]{shen2024nemoalignerscalabletoolkitefficient}
Shen, G., Wang, Z., Delalleau, O., Zeng, J., Dong, Y., Egert, D., Sun, S.,
  Zhang, J., Jain, S., Taghibakhshi, A., Ausin, M.~S., Aithal, A., and
  Kuchaiev, O. (2024).
\newblock Nemo-aligner: Scalable toolkit for efficient model alignment.

\bibitem[Sheng et~al., 2024]{sheng2024hybridflow}
Sheng, G., Zhang, C., Ye, Z., Wu, X., Zhang, W., Zhang, R., Peng, Y., Lin, H.,
  and Wu, C. (2024).
\newblock Hybridflow: A flexible and efficient rlhf framework.
\newblock {\em arXiv preprint arXiv:2409.19256}.

\bibitem[Shoeybi et~al., 2019]{shoeybi2019megatron}
Shoeybi, M., Patwary, M. M.~A., Puri, R., LeGresley, P., Casper, J., and
  Catanzaro, B. (2019).
\newblock Megatron-lm: Training multi-billion parameter language models using
  model parallelism.
\newblock {\em arXiv preprint arXiv:1909.08053}.

\bibitem[Silver et~al., 2017a]{alphazero}
Silver, D., Hubert, T., Schrittwieser, J., Antonoglou, I., Lai, M., Guez, A.,
  Lanctot, M., Sifre, L., Kumaran, D., Graepel, T., et~al. (2017a).
\newblock Mastering chess and shogi by self-play with a general reinforcement
  learning algorithm.
\newblock {\em arXiv preprint arXiv:1712.01815}.

\bibitem[Silver et~al., 2017b]{alphagozero}
Silver, D., Schrittwieser, J., Simonyan, K., Antonoglou, I., Huang, A., Guez,
  A., Hubert, T., Baker, L., Lai, M., Bolton, A., et~al. (2017b).
\newblock Mastering the game of go without human knowledge.
\newblock {\em nature}, 550(7676):354--359.

\bibitem[Stiennon et~al., 2020]{Stiennon2020}
Stiennon, N., Ouyang, L., Wu, J., and et~al. (2020).
\newblock Learning to summarize with human feedback.
\newblock {\em Advances in Neural Information Processing Systems (NeurIPS)}.

\bibitem[Team et~al., 2023]{team2023gemini}
Team, G., Anil, R., Borgeaud, S., Alayrac, J.-B., Yu, J., Soricut, R.,
  Schalkwyk, J., Dai, A.~M., Hauth, A., Millican, K., et~al. (2023).
\newblock Gemini: a family of highly capable multimodal models.
\newblock {\em arXiv preprint arXiv:2312.11805}.

\bibitem[Tian et~al., 2019]{tian2019elf}
Tian, Y., Ma, J., Gong, Q., Sengupta, S., Chen, Z., Pinkerton, J., and Zitnick,
  L. (2019).
\newblock Elf opengo: An analysis and open reimplementation of alphazero.
\newblock In {\em International conference on machine learning}, pages
  6244--6253. PMLR.

\bibitem[Touvron et~al., 2023]{Touvron2023Llama2}
Touvron, H., Martin, L., Stone, K., and et~al. (2023).
\newblock Llama 2: Open foundation and fine-tuned chat models.
\newblock {\em arXiv preprint arXiv:2307.09288}.

\bibitem[Uesato et~al., 2022]{uesato2022solving}
Uesato, J., Kushman, N., Kumar, R., Song, F., Siegel, N., Wang, L., Creswell,
  A., Irving, G., and Higgins, I. (2022).
\newblock Solving math word problems with process-and outcome-based feedback.
\newblock {\em arXiv preprint arXiv:2211.14275}.

\bibitem[Vinyals et~al., 2019]{vinyals2019alphastar}
Vinyals, O., Babuschkin, I., Chung, J., Mathieu, M., Jaderberg, M., Czarnecki,
  W.~M., Dudzik, A., Huang, A., Georgiev, P., Powell, R., et~al. (2019).
\newblock Alphastar: Mastering the real-time strategy game starcraft ii.
\newblock {\em DeepMind blog}, 2:20.

\bibitem[Wang et~al., 2016]{wang2016sample}
Wang, Z., Bapst, V., Heess, N., Mnih, V., Munos, R., Kavukcuoglu, K., and
  De~Freitas, N. (2016).
\newblock Sample efficient actor-critic with experience replay.
\newblock {\em arXiv preprint arXiv:1611.01224}.

\bibitem[Wei et~al., 2025]{wei2025swe}
Wei, Y., Duchenne, O., Copet, J., Carbonneaux, Q., Zhang, L., Fried, D.,
  Synnaeve, G., Singh, R., and Wang, S.~I. (2025).
\newblock Swe-rl: Advancing llm reasoning via reinforcement learning on open
  software evolution.
\newblock {\em arXiv preprint arXiv:2502.18449}.

\bibitem[Xiao et~al., 2023]{flexrlhf}
Xiao, Y., Zhou, Z., Mao, F., Wu, W., Zhao, S., Ju, L., Liang, L., Zhang, X.,
  and Zhou, J. (2023).
\newblock An adaptive placement and parallelism framework for accelerating rlhf
  training.
\newblock {\em arXiv preprint arXiv:2312.11819}.

\bibitem[Xu et~al., 2024]{xu2024perfect}
Xu, T., Helenowski, E., Sankararaman, K.~A., Jin, D., Peng, K., Han, E., Nie,
  S., Zhu, C., Zhang, H., Zhou, W., Zeng, Z., He, Y., Mandyam, K., Talabzadeh,
  A., Khabsa, M., Cohen, G., Tian, Y., Ma, H., Wang, S., and Fang, H. (2024).
\newblock The perfect blend: Redefining rlhf with mixture of judges.

\bibitem[Yao et~al., 2023]{yao2023deepspeedchateasyfastaffordable}
Yao, Z., Aminabadi, R.~Y., Ruwase, O., Rajbhandari, S., Wu, X., Awan, A.~A.,
  Rasley, J., Zhang, M., Li, C., Holmes, C., Zhou, Z., Wyatt, M., Smith, M.,
  Kurilenko, L., Qin, H., Tanaka, M., Che, S., Song, S.~L., and He, Y. (2023).
\newblock Deepspeed-chat: Easy, fast and affordable rlhf training of
  chatgpt-like models at all scales.

\bibitem[Zhang et~al., 2024]{zhang2024scaling}
Zhang, B., Liu, Z., Cherry, C., and Firat, O. (2024).
\newblock When scaling meets llm finetuning: The effect of data, model and
  finetuning method.
\newblock {\em arXiv preprint arXiv:2402.17193}.

\bibitem[Zhao et~al., 2023]{zhao2023pytorchfsdpexperiencesscaling}
Zhao, Y., Gu, A., Varma, R., Luo, L., Huang, C.-C., Xu, M., Wright, L.,
  Shojanazeri, H., Ott, M., Shleifer, S., Desmaison, A., Balioglu, C., Damania,
  P., Nguyen, B., Chauhan, G., Hao, Y., Mathews, A., and Li, S. (2023).
\newblock Pytorch fsdp: Experiences on scaling fully sharded data parallel.

\bibitem[Ziegler et~al., 2019]{ziegler2019fine}
Ziegler, D.~M., Stiennon, N., Wu, J., Brown, T.~B., Radford, A., Amodei, D.,
  Christiano, P., and Irving, G. (2019).
\newblock Fine-tuning language models from human preferences.
\newblock {\em arXiv preprint arXiv:1909.08593}.

\end{thebibliography}

\clearpage
\newpage
\beginappendix

\section{A Comparative Discussion with PPO and GRPO} \label{apepndix:off-policy-comparative}

We provide a comparative discussion on the key difference between the off-policy correction applied in \textsc{LlamaRL}, versus the off-policy correction implicit in PPO \citep{Schulman2017PPO} and more recently GRPO \citep{shao2024deepseekmath}. This is of interest because though the forms of importance sampling corrections in both cases look rather similar, their designs are based on subtly different assumptions of the training processes, which we clarify below. We start by introducing the canonical \texttt{PPO clipping}.

\paragraph{\textbf{PPO clipping.}}
On a high level, PPO applies a double-sided clipping to the importance sampling ratio, conditional further on the sign of the advantage function. Its loss is implemented as
\begin{align*}
\min\left(\frac{ \pi\left(y_t\;|\;x,y_{1:t-1}\right)}{\mu\left(y_t\;|\;x,y_{1:t-1}\right)}A_t, \underbrace{\text{clip}\left(\frac{ \pi\left(y_t\;|\;x,y_{1:t-1}\right)}{\mu\left(y_t\;|\;x,y_{1:t-1}\right)}, 1-\epsilon,1+\epsilon\right)}_{\text{double-sided clipping}}A_t\right),
\end{align*}
where $\epsilon>0$ is a hyper-parameter. The overall design has a few subtle technical properties besides double-sided clipping: (1) \textsc{Gradient clipping.} Most subtly, the losses are implemented such that the gradient is zero when, e.g., the ratio $\frac{ \pi\left(y_t\;|\;x,y_{1:t-1}\right)}{\mu\left(y_t\;|\;x,y_{1:t-1}\right)}$ goes beyond the boundary $[1-\epsilon,1+\epsilon]$ (2) \textsc{Advantage sign.} Besides the double-sided clipping, PPO also conditions the clipping direction on the sign of the estimated advantage function, as evidenced by taking a $\min$ over two loss terms. The main argument here is to construct a lower bound to the policy performance.

We omit further details here and refer interested readers to \citet{Schulman2017PPO} for more extensive discussion. GRPO applies the truncated clipping as PPO, while removing the need for learning critic functions.

\paragraph{\textbf{A deep root in trust region policy optimization for PPO clipping.}}
PPO clipping is deeply rooted in the trust region policy optimization (TRPO) formulation to stabilize RL training \citep{kakade2002approximately,schulman2015trust}, where the motivation is to ensure that the parameter updates are conservative enough to ensure policy improvement. In the context of PPO, the behavior policy $\mu$ takes a particular form: it is the previous iterate of $\pi$, denoted as $\pi_\text{old}$. As a result, the clipping ensures that there is only non-zero update to the policy $\frac{ \pi\left(y_t\;|\;x,y_{1:t-1}\right)}{\pi_\text{old}\left(y_t\;|\;x,y_{1:t-1}\right)}$, when e.g., it is within the range $[1-\epsilon,1+\epsilon]$, that is when $\pi$ does not deviate too much from its previous iterate $\pi_\text{old}$. The TRPO technique has proved effective at stabilizing model updates for large-scale training in a early work \citep{schulman2015trust,Schulman2017PPO}.

\paragraph{\textbf{Key difference from the \textsc{LlamaRL} case.}} Applying the TRPO formulation as-is, we end up with a more synchronous algorithm where we alternate between generating samples from $\pi_\text{old}$ and update the policy to $\pi$. In \textsc{LlamaRL}, due to the asynchronous nature of the training process, the behavior policy $\mu$ is not necessarily the exact previous iterate of $\pi$. For example, $\mu$ can be a few updates away from $\pi$, in addition to being quantized or configured with different sampling parameters than the learner network, among other things. As a result, $\mu$ is not necessarily an ideal policy for trust region regularization. Instead, a perhaps more natural motivation is to correct for the off-policyness of samples generated under $\mu$, to construct approximate on-policy update to $\pi$. This is enabled by clipped importance sampling ratio to balance the bias and variance trade-off \citep{espeholt2018impala}.

Our early ablation results also show that the one-sided clipping in \textsc{LlamaRL} works on par or slightly better than the PPO clipping. However, in principle, both clipping strategies can be combined to synergize the algorithmic strengths of TRPO and asynchronous off-policy learning. 

\section{Lemmas for Theoretical Justification for LlamaRL Speed-up}
We present and prove several Lemmas needed for Theorem \ref{main_theorem}.
\begin{lemma}
\label{constraint_eq}
Let $(b_t^*, b_g^*, m^*)$ be a minimizer to problem \eqref{baseline_prob}. Then it must satisfy
\begin{eqnarray}
\begin{split}
\label{constraint_eq_id}
\frac{(4W_0 + A_tb_t^*) + (W_0 + K_gb_g^*)}{m^*} = M_0.
\end{split}
\end{eqnarray}
\end{lemma}
\begin{proof}[Proof of Lemma \ref{constraint_eq}]
Assume the strict inequality holds
$$\frac{(4W_0 + A_tb_t^*) + (W_0 + K_gb_g^*)}{m^*} < M_0,$$
meaning there is still room to increase $b_t$ or $b_g$. This allows us to further decrease the target
$$\frac{B_0}{G_0}\cdot m \cdot {(}\eta_t + \eta_g{)}$$
due to the monotone decreasing property of $\eta_t, \eta_g$ according to Assumption \ref{ass}, which contradicts the minimizing property of $(b_t^*, b_g^*, m^*)$. Thus concluded Lemma \ref{constraint_eq}.
\end{proof}

\begin{lemma}
\label{constraint_eq2}
Let $(b_t^*, b_g^*, m_t^*, m_g^*)$ be a minimizer to problem \eqref{llamarl_prob}. Then it must satisfy
\begin{eqnarray}
\begin{split}
\label{constraint_eq2_id}
\left\{
\begin{array}{l}
\displaystyle \frac{4W_0}{m_t^*} + \frac{A_tb_t^*}{m_t^*}= M_0, \\
\\
\displaystyle \frac{W_0}{m_g^*} + \frac{K_gb_g^*}{m_g^*}= M_0.
\end{array}
\right.
\end{split}
\end{eqnarray}
\end{lemma}
The proof of Lemma \ref{constraint_eq2} follows a very similar argument as that of Lemma \ref{constraint_eq}, hence omitted. 

\begin{lemma}
\label{decoupled_optim}
    $T^{**}_t$, $T^{**}_g$ (see \eqref{decoupled_sols}) are the minimal values corresponding to the following two optimization problems, respectively:
\begin{eqnarray}
\begin{split}
&\min_{b_t}\quad \eta_t(b_t)m_t\\
&\text{subject to}\quad \frac{4W_0 + A_tb_t}{m_t} \le M_0,
\end{split}
\end{eqnarray}
\begin{eqnarray}
\begin{split}
&\min_{b_g}\quad \eta_g(b_g)m_g\\
&\text{subject to}\quad  \frac{W_0 + K_gb_g}{m_g} \le M_0,
\end{split}
\end{eqnarray}
Furthermore, let $(b_t^{**}, b_g^{**}, m_t^{**}, m_g^{**}, \theta^{**})$ be a solution to \eqref{llamarl_prob}, then the following holds:
\begin{eqnarray}
\begin{split}
\label{decoupled_property}
\left\{
\begin{array}{l}
\displaystyle T^{**}_t = \eta_t^{**}m^{**}_t,\\
\\
\displaystyle T^{**}_g = \eta_t^{**}m^{**}_g,\\
\\
\displaystyle \frac{T^{**}_t}{\theta^{**}} = \frac{T^{**}_g}{1 - \theta^{**}}.
\end{array}
\right.
\end{split}
\end{eqnarray}
\end{lemma}
\begin{proof}[Proof of Lemma \ref{decoupled_optim}]
The first claim follows directly from a similar argument of Lemma \ref{constraint_eq2}'s proof. We omit the details. First, we claim that
\begin{eqnarray}
\begin{split}
\label{balance}
\frac{\eta_t^{**}m_t^{**}}{\theta^{**}} = \frac{\eta_g^{**}m_g^{**}}{1 - \theta^{**}}.
\end{split}
\end{eqnarray}
Suppose the equality fails, then without loss of generality we can assume 
\begin{eqnarray}
\begin{split}
\label{contra}
\frac{\eta_t^{**}m_t^{**}}{\theta^{**}} > \frac{\eta_g^{**}m_g^{**}}{1 - \theta^{**}}.
\end{split}
\end{eqnarray}
It is clear that by slightly increasing $\theta^{**}\longrightarrow \theta^{**} + \epsilon$, the inequality \eqref{contra} is preserved but the following holds:
\begin{eqnarray}
\begin{split}
\max(\frac{\eta_t^{**}m_t^{**}}{\theta^{**}}, \frac{\eta_g^{**}m_g^{**}}{1 - \theta^{**}}) = \frac{\eta_t^{**}m_t^{**}}{\theta^{**}} > \frac{\eta_t^{**}m_t^{**}}{\theta^{**} + \epsilon} = \max(\frac{\eta_t^{**}m_t^{**}}{\theta^{**} + \epsilon}, \frac{\eta_g^{**}m_g^{**}}{1 - \theta^{**} - \epsilon}),
\end{split}
\end{eqnarray}
which contradicts the minimization property of $(b_t^{**}, b_g^{**}, m_t^{**}, m_g^{**}, \theta^{**})$. To prove the first two identities of \eqref{decoupled_property}, we need to show that $\eta_tm_t$ and $\eta_gm_g$ are both minimized at $(b_t^{**}, b_g^{**}, m_t^{**}, m_g^{**}, \theta^{**})$. Again suppose the contrary holds. Note $\eta_tm_t$ and $\eta_gm_g$ are mutually independent of each other, in the sense one can let one of them vary without changing the other. Without loss of generality, let us assume $\eta_t^{**}m_t^{**}$ is not minimal. This implies that there exists $b_t^\prime, m_t^\prime$, such that
\begin{eqnarray}
\begin{split}
\eta^\prime_tm_t^\prime < \eta^{**}_tm_t^{**}.
\end{split}
\end{eqnarray}
Choose $\epsilon > 0$ small enough such that
\begin{eqnarray}
\begin{split}
\label{trainer_new}
\frac{\eta^\prime_tm_t^\prime}{\theta^{**} + \epsilon} < \frac{\eta^{**}_tm_t^{**}}{\theta^{**}}.
\end{split}
\end{eqnarray}
In addition, we have
\begin{eqnarray}
\begin{split}
\label{gen_new}
\frac{\eta^{**}_gm_g^{**}}{1 - \theta^{**} - \epsilon} < \frac{\eta^{**}_gm_g^{**}}{1 - \theta^{**}}.
\end{split}
\end{eqnarray}
Combining \eqref{trainer_new} and \eqref{gen_new}, it follows that
\begin{eqnarray}
\begin{split}
\max(\frac{\eta^\prime_tm_t^\prime}{\theta^{**} + \epsilon}, \frac{\eta^{**}_gm_g^{**}}{1 - \theta^{**} - \epsilon}) < \max(\frac{\eta^{**}_tm_t^{**}}{\theta^{**}}, \frac{\eta^{**}_gm_g^{**}}{1 - \theta^{**}}).
\end{split}
\end{eqnarray}
That is, we found a new point $(b_t^\prime, b_g^{**}, m_t^\prime, m_g^{**}, \theta^{**} + \epsilon)$ that yields a smaller value to problem \eqref{llamarl_prob}, contradicting the minimization property of 
$$(b_t^{**}, b_g^{**}, m_t^{**}, m_g^{**}, \theta^{**}).$$
Thus we proved the first two identities of \eqref{decoupled_property}. The third identity of \eqref{decoupled_property} now follows immediately from \eqref{balance}.
\end{proof}

\end{document}